\documentclass[twoside]{article}

\usepackage[accepted]{aistats2021}
\usepackage[round]{natbib}

\usepackage[T1]{fontenc}
\usepackage{times}
\usepackage{soul}
\usepackage{url}
\usepackage[utf8]{inputenc}
\usepackage[small]{caption}
\usepackage{graphicx}
\usepackage{algorithmic}
\usepackage{float}
\usepackage{amsfonts}
\usepackage{amsmath}
\usepackage{amssymb}
\usepackage{amsthm}
\usepackage{bbm}
\usepackage{booktabs}
\usepackage{enumitem}
\usepackage[ruled]{algorithm2e}
\usepackage[belowskip=-5pt,aboveskip=5pt]{caption}
\usepackage{wrapfig}
\usepackage{bm}
\usepackage{dirtytalk}
\usepackage{thmtools}

\usepackage[usenames,dvipsnames]{xcolor}
\usepackage[bookmarks=false]{hyperref}
\hypersetup{
  pdffitwindow=true,
  pdfstartview={FitH},
  pdfnewwindow=true,
  colorlinks,
  linktocpage=true,
  linkcolor=Green,
  urlcolor=Green,
  citecolor=Green
}
\usepackage[capitalize,noabbrev]{cleveref}

\usepackage[backgroundcolor=white]{todonotes}
\usepackage{tikz}
\usetikzlibrary{bayesnet}
\newcommand{\E}[2]{\mathbb{E}_{#1} \left[#2\right]}

\newcommand{\abs}[1]{\left|#1\right|}

\newcommand{\prob}[1]{\mathbb{P} \left(#1\right)}

\newcommand{\probx}{P^\mathsf{x}}
\newcommand{\probr}{P^\mathsf{r}}
\newcommand{\indicator}[1]{\mathbbm{1}\left[#1\right]}

\DeclareMathOperator*{\argmax}{arg\,max\,}

\newcommand{\cX}{\mathcal{X}}
\newcommand{\cA}{\mathcal{A}}
\newcommand{\cD}{\mathcal{D}}
\newcommand{\cH}{\mathcal{H}}

\newcommand{\cZ}{\mathcal{Z}}

\newcommand{\hatV}{\hat{V}}

\newtheorem{assumption}{Assumption}

\newtheorem{lemma}{Lemma}

\newtheorem{corollary}{Corollary}

\begin{document}


\twocolumn[

\aistatstitle{Non-Stationary Off-Policy Optimization}

\aistatsauthor{Joey Hong \And Branislav Kveton \And  Manzil Zaheer \And Yinlam Chow \And Amr Ahmed }

\aistatsaddress{Google Research}]

\begin{abstract}
Off-policy learning is a framework for evaluating and optimizing policies without deploying them, from data collected by another policy. Real-world environments are typically non-stationary and the offline learned policies should adapt to these changes. To address this challenge, we study the novel problem of off-policy optimization in piecewise-stationary contextual bandits. Our proposed solution has two phases. In the offline learning phase, we partition logged data into categorical latent states and learn a near-optimal sub-policy for each state. In the online deployment phase, we adaptively switch between the learned sub-policies based on their performance. This approach is practical and analyzable, and we provide guarantees on both the quality of off-policy optimization and the regret during online deployment. To show the effectiveness of our approach, we compare it to state-of-the-art baselines on both synthetic and real-world datasets. Our approach outperforms methods that act only on observed context.
\end{abstract}

\section{Introduction}
\label{sec:introduction}

When users interact with online platforms, such as search engines or recommender systems, their behavior is often guided by certain contexts that the system cannot directly observe. Examples of these contexts include \emph{user preferences}, or in shorter term, \emph{user intent}. As the user interacts with the system, these contexts are slowly revealed based on the actions and responses of the user. A good recommender system should be able to utilize these contexts to update the recommendation actions accordingly.

One popular framework to learn recommendation actions conditioned on contexts is using contextual bandits \citep{bandit_book}. In contextual bandits, an agent (or policy) chooses an action based on current contexts and the feedback observed in previous rounds. Contextual bandits have been applied to many core machine learning systems, including search engines, recommender systems, and ad placement \citep{yahoo_contextual,msft_paper}.

Contextual bandit algorithms are either \textit{on-policy}, where the agent learns online from real-world interactions \citep{greedy_side_info,improved_linucb}, or \textit{off-policy}, where the learning process uses offline logged data collected by other policies \citep{log_learning,yahoo_contextual}. While the former is more straightforward, the latter is more suitable for applications where sub-optimal interactions are costly and may lead to costly outcomes.

Most existing contextual bandit algorithms assume that rewards are sampled from a stationary conditional distribution. While this is a valid assumption in simpler problems, where the user intents remain static during interactions, in general the environment should be non-stationary, where user preferences may change during the interactions due to some unexpected events. These shifts in the environment can either be smooth \citep{nonstationary_mab} or abrupt at certain points in time \citep{piecewse_stationary}. Here we mainly focus on the latter case, known as the \textit{piecewise-stationary} environment \citep{piecewse_stationary,sliding_window_ucb}, which is applicable to many event-sensitive decision-making problems.

Non-stationary bandits \citep{exp3_s,nonstationary_contextual_colt}, and more specifically piecewise-stationary bandits \citep{piecewse_stationary,sliding_window_ucb,wmd}, have been studied extensively in the on-policy setting. The prior work in non-stationary off-policy learning only considered policy evaluation, where the evolution of contexts is modeled using time series \citep{arima_ope} or by weighting past observations \citep{jagerman}. Neither of these works considered policy optimization.

In this work, we develop a principled off-policy method to learn a piecewise-stationary contextual bandit policy with performance guarantees. Our algorithm consists of both the offline and online learning phases. In the offline phase, the piecewise-stationarity is modeled with a categorical latent state, whose evolution is either modeled by a change-point detector \citep{cusum,m_ucb} or a hidden Markov model (HMM) \citep{hmm}. At each latent state, a corresponding policy is learned from a subset of offline data associated with that state. With the set of policies learned offline, the online phase then selects which policy to deploy based on a mixture-of-experts \citep{exp3_s,nonstationary_contextual_colt} online learning approach. We derive high-probability bounds on the off-policy performance of the learned policies and also analyze the regret of the online policy deployment. Finally, the effectiveness of our approach is demonstrated in both synthetic and real-world experiments, where we outperform existing off-policy contextual bandit baselines. We address two novel challenges. First, we are the first to consider the bias in off-policy estimation due to an unknown latent state. Second, it is nontrivial to deploy a non-stationary policy learned offline. We are the first to propose a framework for learning the components of a switching policy offline, and then augment them with an adaptive switching algorithm online.

\section{Background}
\label{sec:background}

Let $\cX$ be a set of contexts and $\cA = [K]$ be a set of actions. A typical contextual bandit setting consists of an agent interacting with a stationary environment over $T$ rounds. In round $t \in [T]$, context $x_t \in \cX$ is sampled from an unknown distribution $\probx$. Then, conditioned on $x_t$, the agent chooses an action $a_t \in \cA$. Finally, conditioned on $x_t$ and $a_t$, a reward $r_t \in [0, 1]$ is sampled from an unknown distribution $\probr(\cdot \mid x_t, a_t)$.

Let $\cH = \{ \pi: \cX \to \Delta^{K - 1}\}$ be the set of \emph{stochastic stationary policies}, where $\Delta^{K - 1}$ is the $(K - 1)$-dimensional simplex. We use shorthand $x, a, r \sim P, \pi$ to denote a triplet sampled as $x \sim \probx, a \sim \pi(\cdot \mid x)$, and $r \sim \probr(\cdot \mid x, a)$. We define
\begin{align*}
  \E{x, a, r \sim P, \pi}{r}
  = \mathbb{E}_{x \sim \probx} \mathbb{E}_{a \sim \pi(\cdot \mid x)}
  \E{r \sim \probr(\cdot \mid x, a)}{r}\,.
\end{align*}
With this notation, the expected reward of policy $\pi \in \cH$ in round $t$ can be written
\begin{align*}
  V_t(\pi)
  = \E{x_t, a_t, r_t \sim P, \pi}{r_t}\,.
\end{align*}
Traditionally, $V_t(\pi)$ is the same for all rounds $t$.

In off-policy learning, actions are chosen by a known stationary logging policy $\pi_0 \in \cH$. Logged data are collected in the form of tuples
$$
  \cD
  = \{(x_1, a_1, r_1, p_1), \hdots, (x_T, a_T, r_T, p_T)\}\,,
$$
where $x_t, a_t, r_t \sim P, \pi_0$ and $p_t = \pi_0(a_t \mid x_t)$  is the probability that the logging policy takes action $a_t$ under context $x_t$. For simplicity, we assume that $\pi_0$ is known. Note that if the logging policy is not known, a stationary $\pi_0$ can be estimated from logged data to approximate the true logging policy \citep{log_learning,surrogate_policy,topk_off_policy}. Off-policy learning focuses on two tasks: evaluation and optimization.

\subsection{Off-Policy Evaluation}

The goal is to estimate the expected reward of a target policy $\pi \in \mathcal{H}$, $V(\pi) = \sum_{t = 1}^T V_t(\pi)$, from logged data $\cD$. One popular approach is \emph{inverse propensity scoring (IPS)} \citep{ips}, which reweighs observations with importance weights as
\begin{align*}
  \hatV(\pi)
  = \sum_{t=1}^T \min \left\{ M, \frac{\pi(a_t \mid x_t)}{p_t} \right\} r_t\,,
\end{align*}
where $M$ is a tunable \emph{clipping parameter}. When $M = \infty$, the IPS estimator is unbiased, that is $\mathbb{E}[\hat{V}(\pi)] = V(\pi)$. But its variance could be unbounded if the target and logging policies differ substantially. The clipping parameter $M$ trades off variance due to differences in target and logging policies for bias from underestimating the reward \citep{clipping,msft_paper}. There are methods to design the clipping weight to optimize such trade-offs~\citep{dr,switch}. While we focus on the IPS estimator, our work can be incorporated into other estimators, such as the \emph{direct method (DM)} and \emph{doubly robust (DR)} estimator \citep{dr}, which leverage a reward model $\hat{r}(x, a) \simeq \E{r \sim \probr(\cdot \mid x, a)}{r}$, where $\simeq$ denotes an approximation by fitting on $\cD$.

\subsection{Off-Policy Optimization}

Our goal is to learn a policy with the highest expected reward, $\pi^* = \argmax_{\pi \in \mathcal{H}} V(\pi)$. One popular solution is to maximize the IPS estimate, $\hat{\pi} = \argmax_{\pi \in \cH} \hat{V}(\pi)$~\citep{surrogate_objective}. For stochastic policies, one often optimizes an entropy-regularized estimate ~\citep{surrogate_objective},
\begin{align*}
\hat{\pi} \!=\! \argmax_{\pi \in \cH} \hatV{}(\pi) \!-\! \tau \sum_{t=1}^T \!\sum_{a \in \cA} \pi(a \!\mid\! x_t) \log \pi(a \!\mid\! x_t)\,,
\end{align*}
where $\tau \geq 0$ is the \emph{temperature} parameter that controls the determinism of the learned policy. That is, as $\tau \to 0$, the policy chooses the maximum. Following prior work \citep{poem,norm_poem}, one class of policies that solves this entropy-regularized objective is the linear soft categorical policy $\pi(a \mid x; \theta) \propto \exp(\theta^T f(x, a))$, where $\theta \in \mathbb{R}^d$ is the weight of the linear function approximation w.r.t. the joint feature maps of context and action $f(x, a) \in \mathbb{R}^d$. In the special case of $\cX$ being finite, $f(x, a)$ can be an indicator vector for each pair $(x, a)$, and solving $\hat{\pi}$ reduces to an LP \citep{ranking}.

\section{Setting}
\label{sec:setting}

In non-stationary bandits, the context and reward distributions change with round $t$. 
To model this, we consider an extended contextual bandit setting where the context and reward distributions also depend on a \emph{discrete latent state} $z \in \cZ$, where $\cZ = [L]$ is the set of $L$ latent states. We denote by $z_t \in \cZ$ the latent state in round $t$, and by $z_{1:T} = (z_t)_{t=1}^T \in \cZ^T$ its sequence over the logged data. We consider $z_{1:T}$ to be fixed but unknown. For analysis, we assume that $L$ is known, but relax this assumption and tune $L$ in the experiments. We also assume that the latent state is unaffected by the actions of the agent, a key difference from reinforcement learning (RL). In search engines, for instance, latent states could be different user intents that change over time, such as $\mathcal{Z} = \{\text{news}, \text{shopping}, \hdots\}$.

We can modify our earlier notation to account for the latent state. Let $\probx_z$ and $\probr_z$ be the corresponding context and reward distributions conditioned on $z$. Then the expected reward of policy $\pi$ at round $t$ is $V_t(\pi) = \E{x, a, r \sim P_{z_t}, \pi}{r}$. The relation between all variables can be summarized in a graphical model in \cref{fig:graph}. Revisiting our search engine example, if a system knew that the user shops, it would likely recommend products to buy. So, instead of policies that only act on observed context, we should consider policies that also act according to the latent state. Therefore, we define a new class of policies $\cH^\cZ$, whose members are $\Pi=(\pi_z)_{z \in \cZ}$, and $\pi_z \in \cH$ are individual stationary policies. We define the value of $\Pi$ as 
\begin{align}
\begin{split}
V(\Pi) &= \sum_{z \in \mathcal{Z}} V_z(\pi_z)\,, \\
V_z(\pi_z) &= \sum_{t = 1}^T \indicator{z_t = z}V_t(\pi_z)\,,
\end{split}
\label{eqn:v}
\end{align}
where the latter is the value of $\pi_z$ on the subset of logged data with latent state $z$. Note that calculating $V(\Pi)$ requires knowing $z_{1:T}$; therefore, this quantity is hard to compute in practice, but can still be used to reason about performance.

Prior works on non-stationary bandits either studied environments with \emph{smooth changes} \citep{nonstationary_mab}, or \emph{piecewise-stationary} environments, where the changes are abrupt at a fixed number of unknown \textit{change-points} \citep{piecewse_stationary,sliding_window_ucb}. In this work, we focus on the latter environment. In a piecewise-stationary environment, we additionally denote by $S$ the number of stationary segments in $z_{1:T}$, where the latent state is constant over a segment. We assume that $S \geq L$, as multiple segments can map to the same latent state, and that $S$ is small. We denote the change-points by 
\begin{align}
1 < \tau_1 < \hdots < \tau_{S-1} < T = \tau_S\,,
\label{eqn:changepoints}
\end{align}
where we let $\tau_{S} = T$ to simplify exposition. 

\begin{figure}
\begin{center}
\begin{tikzpicture}[x=1.2cm,y=0.5cm]
  % Nodes
  \node[obs]                 (r) {$r_t$} ; %
  \node[obs, above= of r]    (x) {$x_t$} ; %
  \node[det, right= of x]    (a) {$a_t$} ; %
      
  \node[latent, left=of x] (z) {$z_t$} ; %

  \edge {x} {a};
  \edge {x} {r};
  \edge {a} {r};
  \edge {z} {x};
  \edge {z} {r};

\end{tikzpicture}
\end{center} 
\caption{Graphical model for latent state $z_t$, context $x_t$, action $a_t$, and reward $r_t$.}
\label{fig:graph}
\vspace{-0.1in}
\end{figure}
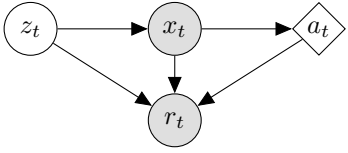

\section{Off-Policy Evaluation}
\label{sec:evaluation}

To extend off-policy learning to the piecewise-stationary latent setting, we consider an IPS estimator
for $\Pi \in \cH^Z$
\begin{align}
  \hatV(\Pi)
  & = \sum_{z \in \mathcal{Z}} \hatV_z(\pi_z)\,,
  \label{eqn:z_ips} \\
  \hatV_z(\pi_z)
  & = \sum_{t = 1}^T \indicator{\hat{z}_t = z}\cdot
  \min \left\{M, \frac{\pi_z(a_t \mid x_t)}{p_t}\right\} r_t\,,
  \nonumber
\end{align} 
where $\hatV_z(\pi_z)$ is the IPS estimator for the logged data with latent state $z$ and $\hat{z}_{1:T}$ is a sequence of latent states predicted by an \emph{oracle} $O$. This estimator partitions the logged data by latent state. 

For simplicity, we restrict our performance analysis to a set of policies where the clipping condition is always satisfied,
\begin{align}
  \cH
  = \left\{\pi: \frac{\pi(a \mid x)}{\pi_0(a \mid x)} \leq M
  \textrm{ for all } a \in \cA, x \in \cX\right\}\,,
  \label{eq:clipped policies}
\end{align}
so that the propensity score does not needed to be clipped. The analysis can be straightforwardly extended to a general policy class, and this only adds an extra bias term to the error bound \citep{clipping,ranking}. We omit this for the sake of brevity.

If the oracle accurately predicts all the ground-truth latent states, i.e., $\hat{z}_{1 : T} = z_{1 : T}$, and if $M=\infty$, then the following lemma shows that the IPS estimator $\hatV(\pi)$ is unbiased.

\begin{lemma}
For any $\Pi \in \cH^Z$, the IPS estimator $\hat{V}(\Pi)$ in \eqref{eqn:z_ips} is unbiased
when $\hat{z}_{1 : T} = z_{1 : T}$.
\label{lem:unbiased_z_ips}
\end{lemma}
\begin{proof}
From definition of $\hat{V}(\Pi)$ in \eqref{eqn:z_ips}, we have
\begin{align*}
  V(\Pi)
  & = \sum_{t=1}^T V_t(\pi_{z_t}) \\
  & = \sum_{t = 1}^T \E{x_t, a_t, r_t \sim P_{z_t}, \pi_0}
  {\frac{\pi_{z_t}(a_t \mid x_t)}{p_t} r_t} \\
  & = \E{}{\sum_{t = 1}^T \frac{\pi_{z_t}(a_t \mid x_t)}{p_t} r_t}
  = \E{}{\hat{V}(\Pi)}\,,
\end{align*}
where the last expectation is over all $x_t, a_t, r_t \sim P_{z_t}, \pi_0$, for any $t \in [T]$.
\end{proof}

While the above technical result justifies our choice of the IPS estimator for piecewise-stationary environments, in reality there is no practical way to ensure a perfect latent state estimation because the latent states $z_{1:T}$ are not observed in logged data $\cD$. To address this challenge, in the following we assume that the latent state oracle $O$ has a low prediction error with high probability and show how this error propagates into off-policy value estimation.

\begin{assumption}
\label{ass:oracle} For any $z_{1 : T}$ and $\delta \in (0, 1]$, oracle $O$ estimates $\hat{z}_{1 : T}$ such that $\sum_{t = 1}^T \mathbbm{1}[\hat{z}_t \neq z_t] \leq \varepsilon(T, \delta)$ holds with probability at least $1 - \delta$, where $\varepsilon(T, \delta) = o(T)$ is some function of $T$ and $\delta$.
\end{assumption}

Now, consider when a latent state prediction is generated by an oracle $O$ that satisfies \cref{ass:oracle}. Using this oracle, we can provide an upper bound on the estimation error (whose proof is in \cref{sec:proof_offline}) of the IPS estimator in \eqref{eqn:z_ips}.

\begin{restatable}[]{lemma}{evalmain}
\label{lem:eval_main} For any policy $\Pi \in \cH^Z$, its IPS estimate $\hatV(\Pi)$ in \eqref{eqn:z_ips}, and true value $V(\Pi)$, we have that
\begin{align*}
 |V(\Pi) - \hatV(\Pi)|
 \leq M \varepsilon(T, \delta_1/2) + M\sqrt{2T \log(4/\delta_2)}
\end{align*}
holds with probability at least $1 - \delta_1 - \delta_2$.
\end{restatable}

This technical lemma shows that in a piecewise-stationary environment, the error of the IPS estimator can be decomposed into the latent oracle prediction error and a statistical error term that is sublinear in $T$. In the rest of this section, we introduce two latent prediction oracles. The first one is based on change-point detection, and we show that it satisfies \cref{ass:oracle}. The second one is based on a \emph{hidden Markov model} (HMM), for which we do not prove an error bound but get better empirical performance.

\subsection{Change-Point Detector}
\label{sec:change_point}

In this section, we propose and analyze a change-point detector oracle that satisfies \cref{ass:oracle}. First, we assume a one-to-one mapping between stationary segments and latent states, or $S = L$. We let $z_{1:T}$ form a non-decreasing sequence of integers that satisfies $z_1 = 1$, $z_T = S$ with $\abs{z_{t + 1} - z_t} \leq 1$, $\forall t \in [T - 1]$, and change-points in \eqref{eqn:changepoints}.
In practice, this could over-segment the offline data, if multiple stationary segments can be modeled by the same latent state. However, this assumption is only used in the analysis.

We also assume that changes are \emph{detectable}. This means that the difference in performance of a stationary logging policy before and after the change-point exceeds some threshold.

\begin{assumption}
For each segment $i \in [S]$ there exists a threshold $\Delta > 0$ such that the difference of values between two consecutive change points is greater than $\Delta$, i.e. $\abs{V_{\tau_i}(\pi_0) - V_{\tau_i - 1}(\pi_0)} \geq \Delta$.
\end{assumption}

Similar assumptions are common in piecewise-stationary bandits, where the state-of-the-art algorithms~\citep{cusum,m_ucb} use an online change-point detector to detect change points and reset the parameters of the bandit algorithm upon a change. In this work, we utilize a similar idea but in an offline off-policy setting. We construct a change-point detector oracle $O$ with window size $w$ and detection threshold $c$ (\cref{alg:oracle}).

\begin{algorithm}[t]
%\algsetup{linenosize=\tiny}
\DontPrintSemicolon
\KwIn{window size $w \in \mathbb{N}$, detection threshold $c \in \mathbb{R}^{+}$, and logged data $\cD$}
\BlankLine
\For{$t \gets w + 1$ \KwTo $T - w + 1$}{%
    $\mu_t^- \gets w^{-1} \sum_{i = t-w}^{t - 1} r_i$ \\
    $\mu_t^+ \gets w^{-1} \sum_{i = t}^{t + w - 1} r_i$
}
Initialize candidates $C \gets \{t \in [T]: \abs{\mu_t^- - \mu_t^+} \geq c\}$ and change-points $\Gamma = \emptyset$\\
\While{$C \neq \emptyset$}{%
  Find change-point
  $\hat{\tau} \gets \argmax_{t \in C}\{\abs{\mu_t^- - \mu_t^+}\}$ \\
  $C \gets C \setminus [\hat{\tau} - 2w, \hat{\tau} + 2w]$ \\
  $\Gamma \gets \Gamma \cup \{\hat{\tau}\}$
}
Order all elements in $\Gamma$ as $1 < \hat{\tau}_1 < \hdots < \hat{\tau}_{S' - 1} < T = \hat{\tau}_{S'}$, where $S' = |\Gamma| + 1$. \\
\For{$t \gets 1$ \KwTo $T$}{%
$\hat{z}_t \gets \min \{i \in [S']: t \leq \hat{\tau}_i\}$
}
\caption{Change-point detector oracle}
\label{alg:oracle}
\end{algorithm}

At a high level, $O$ computes difference statistics for each round in the offline data. Then it iteratively selects the round with the highest statistic, declares it a change-point, and removes any nearby rounds from consideration. This continues until there is no statistic that lies above threshold $c$. In the following, we state a latent prediction error bound for this oracle, which is derived in \cref{sec:proof_cd}.

\begin{restatable}[]{theorem}{cdoracle}
\label{thm:cd_oracle} Let $\tau_i - \tau_{i - 1} > 4w$ for all $i \in [S]$. Then for any $\delta \in (0, 1]$, and $c$ and $w$ in \cref{alg:oracle} such that
$$
  \Delta / 2
  \geq c
  \geq \sqrt{2 \log(8 T / \delta) / w}\,,
$$
\cref{alg:oracle} estimates $\hat{z}_{1:T}$ so that $\sum_{t = 1}^T \mathbbm{1}[\hat{z}_t \neq z_t] \leq S w$ holds with probability at least $1 - \delta$.
\end{restatable}

\cref{thm:cd_oracle} says that the oracle $O$ can correctly detect, without false positives, change-points within a window $w$ with high probability. Note that both $w$ and $c$ in \cref{thm:cd_oracle} depend on $\Delta$, which may not be known. A lower bound on $\Delta$, which we denote by $\tilde{\Delta}$, would suffice and may be known. We do this to choose $c$ in the experiments in \cref{sec:experiments}.

\subsection{Hidden Markov Model}
\label{sec:graphical}

Another natural way of partitioning the data is using a latent variable model. In this work, we specifically model the temporal evolution of $z_{1:T}$ with a HMM over $\mathcal Z$ \citep{hmm}. Let $\Phi = [\Phi_{i, j}]_{i, j=1}^L$ be the \emph{transition matrix} with $\Phi_{i, j} = P(z_t = j \mid z_{t - 1} = i)$, and $P_0$ be the \emph{initial distribution} over $\mathcal{Z}$ with $P_{0, i} = P(z_1 = i)$. The latent states evolve according to $z_1 \sim P_0$, and $z_{t + 1} \sim \Phi_{z_t, :}$.  Recall from \cref{sec:background} that we have joint feature maps of context and action $f(x, a) \in \mathbb{R}^d$. We assume the rewards are sampled according to the conditional distribution $P(\cdot \mid x, a, z) = \mathcal{N}(\beta_{z}^T f(x, a), \sigma^2)$, where $\beta = (\beta_z)_{z\in \mathcal{Z}}$ are regression weights. Though we use Gaussians, any distribution could be incorporated. Let $\mathcal{M} = \{P_0, \Phi, \beta\}$ be the HMM parameters. The HMM can be estimated through expectation-maximization (EM) \citep{hmm}. 

Oracle $O$ can use the estimated HMM $\hat{\mathcal{M}}$ to predict $\hat{z}_{1:T}$ from \cref{alg:hmm_oracle}. At each round $t$, the oracle estimates forward and backward probabilities 
\begin{align*}
  A_t(z)
  & = P(x_{1:t}, a_{1:t}, r_{1:t}, z_t = z; \hat{\mathcal{M}})\,, \\
  B_t(z)
  & = P(x_{t+1:T}, a_{t+1:T}, r_{t+1:T} \mid z_t = z; \hat{\mathcal{M}})\,,
\end{align*}
and posterior $Q_t(z) = P(z_t = z \mid x_{1:T}, a_{1:T}, r_{1:T}; \hat{\mathcal{M}})$ using forward-backward recursion \citep{hmm}. Then $O$ predicts $\hat{z}_t = \max_{z \in \mathcal{Z}} Q_t(z)$ at each round $t$. Though the described HMM oracle is practical, currently no guarantees similar to \cref{ass:oracle} can be derived. An analysis similar to \cref{thm:cd_oracle} would require parameter recovery guarantees for the HMM, which to our knowledge, do not exist for EM or spectral methods\footnote{Guarantees exist only on the marginal probability of data.} \citep{spectral_hmm}. Nevertheless, the HMM oracle has several appealing properties. First, unlike the change-point detector, the HMM can map multiple stationary segments into a single latent state, which potentially reduces the size of the latent space. Second, the learned reward model $\hat{r}_z(x, a) = \hat{\beta}_z^T f(x, a) \simeq \E{r \sim \probr_z(\cdot \mid x, a)}{r}$ can be incorporated into more advanced off-policy estimators, such as the DR estimator in \cref{sec:background}, and further reduce variance.

\begin{algorithm}[t]
\caption{HMM oracle}\label{alg:hmm_oracle}
\KwIn{estimated HMM parameters $\hat{\mathcal{M}} = \{\hat{P}_0, \hat{\Phi}, \hat{\beta}\}$ and logged data $\cD$
}
\BlankLine
Initialize $A_0(z) \leftarrow \hat{P}_{0, z}, B_T(z) \leftarrow 1$ \\
\For {$z \in \mathcal{Z}$}{
      Compute $A_t(z), B_t(z)$ for all $t = 1, \hdots, T$ by forward-backward recursion
      \begin{align*}
      &A_t(z) \leftarrow
      \sum_{z' \in \cZ} A_{t-1}(z')
      P(z \mid z'; \hat{\Phi}) P(r_t \mid x_t, a_t, z; \hat{\beta}) \\
      &B_t(z) \leftarrow \\
      &\,\sum_{z' \in \cZ} P(z' \mid z; \hat{\Phi}) P(r_{t+1} \mid x_{t+1}, a_{t+1}, z'; \hat{\beta})
      B_{t + 1}(z')
      \end{align*}
     }
\For {$t \gets 1, 2, \hdots, T$}{%
      Compute $Q_t(z) \propto  A_t(z)B_t(z)$ for all $z \in \mathcal{Z}$ and \\
      $\hat{z}_t \leftarrow \arg\max_{z \in \mathcal{Z}} Q_t(z)$
}
\end{algorithm}

\section{Optimization and Deployment}
\label{sec:optimization and deployment}

We propose a piecewise-stationary off-policy optimization algorithm, which has two parts: (i) an offline optimization that solves for the latent-space policy $\hat{\Pi}=(\hat{\pi}_z)_{z \in \mathcal{Z}}$ where $\hat{\pi}_z = \pi(\cdot|\cdot;\hat{\theta}_z) \in \cH$; and (ii) an online sub-policy selection procedure. We also derive a lower bound on the reward of the policy from offline optimization and an upper bound on the regret of its online deployment.

\begin{algorithm}[t]
\algsetup{linenosize=\tiny}
\DontPrintSemicolon
\KwIn{number of latent states $L \in \mathbb{N}$, 
logged data $\cD$, and oracle $O$
}
\BlankLine
Run $O$ on $\cD$ to get latent state estimates $\hat{z}_{1:T} \in [L]^T$ \\
\For{$z \gets 1$ \KwTo $L$}{%
Solve for 

$\hat{\pi}_z = \argmax_{\pi \in \cH} \sum_{t = 1}^T \indicator{\hat{z}_t = z} \hat{V}_t(\pi)$ \\
}
\caption{Piecewise off-policy learning}
\label{alg:policy_learning}
\end{algorithm}

\begin{algorithm}[t]
\algsetup{linenosize=\tiny}
\DontPrintSemicolon
\KwIn{learned policy $\hat{\Pi} \in \cH^{\mathcal{Z}}$ and
mixture-of-experts algorithm $\mathcal{E}$
}
\BlankLine
Initialize algorithm $\mathcal{E}_1$ \\
\For {$t \gets 1$ \KwTo $T$}{%
    Given $x_t$, choose action $a_t \sim \mathcal{E}_t(x_t, \hat{\Pi})$ \\
    Update $\mathcal{E}_{t+1}$ from $\mathcal{E}_t$ with reward $r_t \sim \probr_{z_t}(\cdot \mid x_t, a_t)$
}
\caption{Piecewise policy deployment}
\label{alg:policy_deployment}
\end{algorithm}

\subsection{Off-Policy Optimization}
\label{sec:optimization}

For optimization, we leverage the fact that logged data are partitioned into $L$ sub-datasets, each corresponding to a particular latent state, which gives the IPS estimator $\hatV(\pi)$ in \eqref{eqn:z_ips} a separable structure. In this way, policy optimization can be broken down into learning the best policy at each individual latent state $z$. Formally, each component of $\hat{\Pi}$ is learned by solving the optimization $\hat{\pi}_z = \argmax_{\pi \in \cH} \hatV_z(\pi)$. 

If each sub-policy $\hat{\pi}_z = \pi(\cdot \mid \cdot;\hat{\theta}_z)$ is parameterized by some $\hat{\theta}_z \in \Theta$, where $\Theta$ denotes the space of model parameters, then we solve the following for each latent state $z$
\begin{equation}\label{eqn:train_obj} \hspace{-5pt}
\hat{\theta}_z 
= \argmax_{\theta \in \Theta} \sum_{t = 1}^T \indicator{\hat{z}_t = z} \cdot \min\left\{M, \frac{\pi(a_t \mid x_t; \theta)}{p_t} \right\}r_t \,.
\end{equation}
If $\hat{\pi}_z$ was a linear soft categorical policy, its parameters $\hat{\theta}_z$ could be found as discussed in \cref{sec:background}. Otherwise, following prior work \citep{poem}, we can iteratively solve for each sub-policy using off-the-shelf gradient ascent algorithms. \cref{alg:policy_learning} summarizes our approach to learning $\hat{\Pi}$.

For $\hat{\Pi} = \arg\max_{\Pi \in \cH^{\mathcal{Z}}} \hat{V}(\Pi)$, we now bound from below the expected reward of $\hat{\Pi}$, in terms of any oracle $O$ that satisfies \cref{ass:oracle}. We merely state the result here and defer its derivation to \cref{sec:proof_offline}.

\begin{restatable}[]{theorem}{optmain}
\label{thm:main} Let 
$$\hat{\Pi}
  = \argmax_{\Pi \in \cH^{\mathcal{Z}}}\hat{V}(\Pi)\,, \quad
  \Pi^*
  = \argmax_{\Pi \in \cH^{\mathcal{Z}}} V(\Pi)
$$
be the optimal latent policies w.r.t. the off-policy estimated value and the true value, respectively. Then for any $\delta_1, \delta_2 \in (0, 1]$, we have that
\begin{align*}
  V(\hat{\Pi}) 
  \geq
  V(\Pi^*) - 2M \varepsilon(T, \delta_1/2) - 2M\sqrt{2T \log(4/\delta_2)}
\end{align*}
holds with probability at least $1 - \delta_1 - \delta_2$.
\end{restatable}

\cref{thm:main} states that the reward gap of the learned policy $\hat{\Pi}$ from $\Pi^*$ decomposes into the error due to oracle $O$ and randomness of logged data $\cD$. It is important to note that we assume the true latent sequence $z_{1:T}$ is known when measuring the performance of a policy. This is evident in \eqref{eqn:v}, where sub-policy used for round $t$ is $z_t$. We relax this assumption in \cref{sec:deployment}, where the latent state is estimated only from past interactions.

Next we derive a lower bound on expected reward of policy $\hat{\Pi}$ learned by \cref{alg:policy_learning} with change-point detector oracle $O$ in \cref{alg:oracle}.   

\begin{corollary}
\label{cor:main} Fix any $\tilde{\Delta} \leq \Delta$ and $\delta_1, \delta_2 \in (0, 1]$. Let oracle $O$ be Algorithm~\ref{alg:oracle} with
$$
  w 
  = 8 \log(16 T / \delta_1) / \tilde{\Delta}^2\,, \quad 
  c
  = \tilde{\Delta}/2\,,
$$
and $\Pi^\ast$, $\hat{\Pi}$ be defined as in \cref{thm:main}. Then
\begin{align*}\hspace{-5pt}
  V(\hat{\Pi}) \geq {} &
  V(\Pi^*) - 16 M \left(S \log(16 T / \delta_1) / \tilde{\Delta}^2\right) - {} \\
  & 2 M \sqrt{2T \log(4 / \delta_2)}
\end{align*}
holds with probability at least $1 - \delta_1 - \delta_2$.
\end{corollary}

\cref{cor:main} follows from combining \cref{thm:cd_oracle,thm:main}. It says that if the estimated latent states $\hat{z}_{1:T}$ are generated by \cref{alg:oracle}, and the policy $\hat{\Pi}$ is learned by \cref{alg:policy_learning}, then the difference in the expected rewards of $\hat{\Pi}$ from $\Pi^\ast$ is $O(\log{T}\sqrt{T})$.

\subsection{Online Deployment}
\label{sec:deployment}

Recall that our offline optimizer learns a vector of sub-policies $\hat{\Pi} = (\hat{\pi}_z)_{z \in \mathcal{Z}}$, one for each latent state. During online deployment, however, the latent state is still unobserved, and we cannot query an oracle as we did offline. We need an online algorithm that switches between the learned sub-policies based on past rewards.

Our solution is to treat each sub-policy as an \say{expert}, and select which one to execute in each round by a mixture-of-experts algorithm $\mathcal{E}$. This is because the online performance of sub-policies can be treated as a surrogate predictor for the unknown latent state. Our online algorithm is presented in \cref{alg:policy_deployment}, and takes a mixture-of-experts algorithm $\mathcal{E}$ as an input. At each round $t$, actions are sampled as $a_t \sim \mathcal{E}_t(x_t, \hat{\Pi})$, where $\mathcal{E}_t$ depends on the history of rewards thus far and context $x_t$.

To simplify exposition, we introduce shorthand $\E{z, \pi}{\cdot} = \E{x, a, r \sim P_z, \pi}{\cdot}$. We also assume initially that the online latent sequence is the same as $z_{1:T}$ in offline data; we later give a high-level argument on how to relax this assumption.
Let the $T$-round regret be defined as
$$
    \mathcal{R}(T; \mathcal{E}, \hat{\Pi}) 
    = \sum_{t = 1}^T \E{z_t, \pi^*_{z_t}}{r_t} - \sum_{t = 1}^T \E{z_t, \mathcal{E}_t}{r_t}\,.
$$
The first term is the optimal policy $\Pi^*$ acting according to the true latent state. The second term is our offline-learned policy $\hat{\Pi}$ acting according to $\mathcal{E}$. In this section, we give a brief outline of how to bound the online regret, and defer details to \cref{sec:proof_online}.

Recall that $S$ is the number of stationary segments, and change-points are defined as in \eqref{eqn:changepoints}. Assuming the latent state is constant over a stationary segment, we first have the following lemma that decomposes the regret $\mathcal{R}(T; \mathcal{E}, \hat{\Pi})$.

\begin{restatable}[]{lemma}{regretdecomposition}
The regret $\mathcal{R}(T; \mathcal{E}, \hat{\Pi})$ is bounded from above as
\begin{align}
\begin{split}
    &\mathcal{R}(T; \mathcal{E}, \hat{\Pi}) 
    \leq \left[\sum_{t = 1}^T \E{z_t, \pi^*_{z_t}}{r_t}
    - \sum_{t = 1}^T \E{z_t, \hat{\pi}_{z_t}}{r_t}\right] \\
    &\quad + \left[\sum_{s = 1}^S  \max_{z \in \mathcal{Z}} \sum_{t = \tau_{s-1}}^{\tau_s - 1} \E{z_t, \hat{\pi}_z}{r_t}
     - \sum_{t = 1}^T \E{z_t, \mathcal{E}_t}{r_t}\right]\,.
\end{split}
\end{align}
\label{lem:regret_decomposition}
\end{restatable}

The first-term is exactly $V(\Pi^*) - V(\hat{\Pi})$ and is bounded by \cref{thm:main} in our offline analysis, which shows near-optimality of $\hat{\Pi}$ when $z_{1:T}$ are known.
The second term is bounded by the regret of mixture-of-experts algorithm $\mathcal{E}$ over $S - 1$ change-points. 

Prior work showed an optimal $T$-round switching regret with $S-1$ switches of $O(\sqrt{SKT})$ \citep{nonstationary_contextual_colt}. One such algorithm that is optimal up to log factors is Exp4.S \citep{nonstationary_contextual_colt}. We adapt Exp4.S to stochastic experts in \cref{alg:exp4_s} in \cref{sec:proof_online}. Using this algorithm for $\mathcal{E}$ gives us the following regret bound.

\begin{restatable}[]{theorem}{mainregret}
\label{thm:online_regret} Let $\hat{\Pi}$ be defined as in \cref{thm:main} and $\mathcal{E}$ be Exp4.S (\cref{alg:exp4_s}). Let $z_{1:T}$ be the same latent states as in offline data $\cD$ and $S$ be the number of stationary segments. Then for any $\delta_1, \delta_2 \in (0, 1]$, we have that
\begin{align*}
   &\mathcal{R}(T; \mathcal{E}, \hat{\Pi}) \leq \\
   &\quad 2M \varepsilon(T, \delta_1/2) + 2M\sqrt{2T \log(4/\delta_2)} + 2\sqrt{STK\log{L}}
\end{align*}
holds with probability at least $1 - \delta_1 - \delta_2$.
\end{restatable}

The regret of deploying our offline-learned policy $\hat{\Pi}$ online elegantly decomposes into the expected reward gap of $\hat{\Pi}$ from $\Pi^\ast$ in off-policy optimization, and the regret of $\mathcal{E}$ that switches between sub-policies of $\hat{\Pi}$.

\subsection{Policy Selection by Posterior Sampling}

In \cref{sec:graphical}, we learned an HMM offline to identify the latent states. The same HMM can be used to sample a latent state from its posterior probability, and act according to the corresponding expert, similarly to Bayesian policy reuse for adversarial environments \citep{bpr}. Some guarantees exist for posterior sampling of stationary latent states \citep{latent_bandits_revisited}, but not for ones that evolve according to an unknown HMM. Our posterior sampling algorithm is in \cref{alg:hmm_ts}, and works by computing a latent state posterior
$Q_t(z) = P(z_t = z \mid x_{1:t-1}, a_{1:t-1}, r_{1:t-1}; \hat{\mathcal{M}})$. 
Note that this is different to $Q_t$ defined in \cref{sec:graphical}, because we only condition on the history.
\cref{alg:hmm_ts} can be used as $\mathcal{E}$ in \cref{alg:policy_deployment} if an HMM was estimated offline. While regret guarantees do not exist as for Exp4.S, such posterior sampling algorithms typically have much better empirical performance. 

\begin{algorithm}[t]
\caption{HMM posterior sampling}\label{alg:hmm_ts}
\KwIn{vector of experts $\hat{\Pi} \in \cH^{\cZ}$ and
estimated HMM parameters $\hat{\mathcal{M}} = \{\hat{P}_0, \hat{\Phi}, \hat{\beta}\}$
}
\BlankLine
Initialize $w_1 = \hat{P}_0$. \\
\For {$t \gets 1, 2, \hdots, T$}{%
    Observe $x_t \in \cX$, and expert feedback $\hat{\pi}_z(\cdot \mid x_t), \, \forall z \in \mathcal{Z}$ \\

    Choose action $a_t \sim w_t$, where for each $a \in \cA$,
    $w_t(a) = \sum_{z \in \mathcal{Z}} Q_t(z) \hat{\pi}_z(a \mid x_t)$ 
    
    Observe $r_t$
    
    Update the latent-state posterior distribution, $\forall z \in \mathcal{Z}$,
    \begin{align*}\hspace{-5pt}
    Q_{t+1}(z) \propto  \sum_{z' \in \mathcal{Z}} Q_t(z') P(r_t \mid x_t, a_t, z'; \hat{\beta}) P(z \mid z'; \hat{\Phi})
    \end{align*}
}
\end{algorithm}

\subsection{Extension to Different Latent Sequences} 

In \cref{thm:online_regret}, we bound the online regret of our algorithm on latent state sequence $z_{1 : T}$ in \cref{lem:regret_decomposition}. Specifically, the first term of the regret decomposition given in \cref{lem:regret_decomposition} is $V(\Pi_*) - V(\hat{\Pi})$, which is computed with respect to $z_{1:T}$.

Now we consider online data with a different latent state sequence $z'_{1:T}$. For stationary policy $\pi \in \cH$, we denote its value in round $t$ by $V'_t(\pi) = \E{x, a, r \sim P_{z'_t}, \pi}{r}$. For policy $\Pi \in \cH^\cZ$, we define its value by $V'(\Pi) = \sum_{z \in \mathcal{Z}} V'_z(\pi_z)$ where $V'_z(\pi_z) = \sum_{t = 1}^T \indicator{z'_t = z} V'_t(\pi_z)$.
We want to characterize how the reward gap $V'(\Pi_*) - V'(\hat{\Pi})$ changes when computed with respect to $z'_{1:T}$.

For $z \in \mathcal{Z}$, let $T_z$ and $T'_z$ be the number of occurrences of $z$ in $z_{1:T}$ and $z'_{1:T}$, respectively. Note that
\begin{align*}
  V_z(\pi_z)
  = \sum_{t = 1}^T \indicator{z_t = z} V_t(\pi_z) = T_z \E{x, a, r \sim P_z, \pi_z}{r}\,,
\end{align*}
and similarly for $V'_z$, as the value of any policy under latent state $z$ is constant. We can bound the difference in reward gap of $\hat{\Pi}$ between the two latent sequences as
\begin{align*}
&\left(V'(\Pi^*) - V'(\hat{\Pi})\right) - \left(V(\Pi^*) - V(\hat{\Pi})\right) \\
&= \sum_{z \in Z} \left(V'_z(\pi^*_z) - V'_z(\hat{\pi}_z)\right) - \left(V_z(\pi^*_z) - V_z(\hat{\pi}_z)\right) \\
&\leq \sum_{z \in Z} \abs{T'_z  - T_z}
\leq \sqrt{L\textstyle\sum_{z \in \mathcal{Z}} (T'_z - T_z)^2}\,.
\end{align*}
where the first inequality is due to naively bounding from above $\E{x, a, r \sim P_z, \pi^*_z}{r} - \E{x, a, r \sim P_z, \hat{\pi}_z}{r} \leq 1$, and the second bounds the $\ell_1$ with $\ell_2$-norm. This additional error can be added to the regret bound in \cref{thm:online_regret}.

\section{Experiments}
\label{sec:experiments}

In this section, we evaluate our approach on synthetic and real-world datasets, and show that it outperforms learning a single stationary policy. We compare the following methods: (i) \textbf{IPS}: a single policy trained on the IPS objective; (ii) \textbf{DR}: a single policy trained on the DR objective, with reward model $\hat{r}(x, a) = \hat{\beta}^T f(x, a)$ fit using least squares; (iii) \textbf{POEM}: a single policy trained on the counterfactual risk minimization (CRM) objective, which adds an empirical covariance regularizer to the objective in \cref{sec:background} \citep{poem}; (iv) \textbf{$k$-CD}: $k$ sub-policies trained using our method with a change-point detector (\cref{alg:oracle}), deployed using Exp4.S (\cref{alg:exp4_s} of \cref{sec:proof_online}); (v) \textbf{$k$-HMM}: $k$ sub-policies trained using our method with an HMM (\cref{alg:hmm_oracle}), deployed using posterior sampling (\cref{alg:hmm_ts}). The first three are baselines in stationary off-policy optimization, and the last two are our approach. In our approach, $k$ is a tunable parameter that estimates the unknown number of latent states $L$. In $k$-CD, we control the number of latent states by $k$-means clustering on detected stationary segments. Specifically, we compute $V(\pi_0)$ the value of the logging policy across each stationary segment, and segments with similar value are clustered into one of $k$ latent states.

\subsection{Synthetic Dataset}
\label{sec:synthetic dataset}

The first problem is a synthetic non-stationary bandit without context, with $\cA = [5]$ and $\mathcal{Z} = [5]$. The mean rewards of actions are sampled uniformly at random as $\mu(a, z) \sim \mathsf{Uniform}(0, 1)$ for each $a \in \cA, z \in \mathcal{Z}$. The rewards are drawn i.i.d.\ as $r \sim \mathcal{N}(\cdot \mid \mu(a, z), \sigma^2)$ with $\sigma = 0.5$. The horizon is $T = 100,000$ rounds. The latent state sequence $z_{1 : T}$ is generated as follows. We set $z_1 = 1$. Then, every $10,000$ rounds, the latent state is incremented by one. After round $50,000$, the latent state is decremented by one every $10,000$ rounds. This is a piecewise-stationary with changes every $10,000$ rounds. Since this problem is non-contextual, the feature vector $f(x, a) \in \{0, 1\}^{|\cA|}$ for action $a$ is its indicator. The logging policy $\pi_0$ is designed to perform well on average over all latent states, which often happens in practice. We define it as $\pi_0(a) \propto \exp(\tilde{\mu}(a))$, where $\tilde{\mu}(a) = \abs{\cZ}^{-1} \sum_{z \in \cZ} \mu(a, z) + \epsilon$ and $\epsilon \sim \mathcal{N}(0, 0.1)$ is a perturbed mean reward for action $a$.

The learned policies are evaluated by a simulated online deployment, on the same latent state sequence $z_{1:T}$ as in logged data. This is the case that we analyze. We relax this assumption in the next experiment. For the change-point detector of $k$-CD, we set $w = 4,000$ and set $c = \sqrt{2 \log(8T^2) / w}$ so that the inequality in \cref{thm:cd_oracle} is satisfied with $\delta = 1/T$. \cref{fig:sim_plot} shows expected rewards of all learned policies. Both of our approaches, $k$-CD and $k$-HMM, significantly outperform learning a stationary policy, with $k$-HMM performing better. This is likely because $k$-HMM acts stochastically according to the learned HMM, whereas $k$-CD, which uses adversarial Exp4.S, acts too conservatively. Since the number of latent states $L$ is not known in practice, it must be estimated, and we also do that in this experiment. This results in a bias-variance trade-off, where underestimating $k < L$ leads to under-partitioned data and biased sub-policies, and overestimating $k > L$ results in over-partitioned data and sub-policies with high variance. This is evident in \cref{fig:sim_plot}, as both result in suboptimal performance compared to choosing $k = L$.

\begin{figure}
\begin{center}
\begin{minipage}{0.32\textwidth}
    \includegraphics[width=\textwidth]{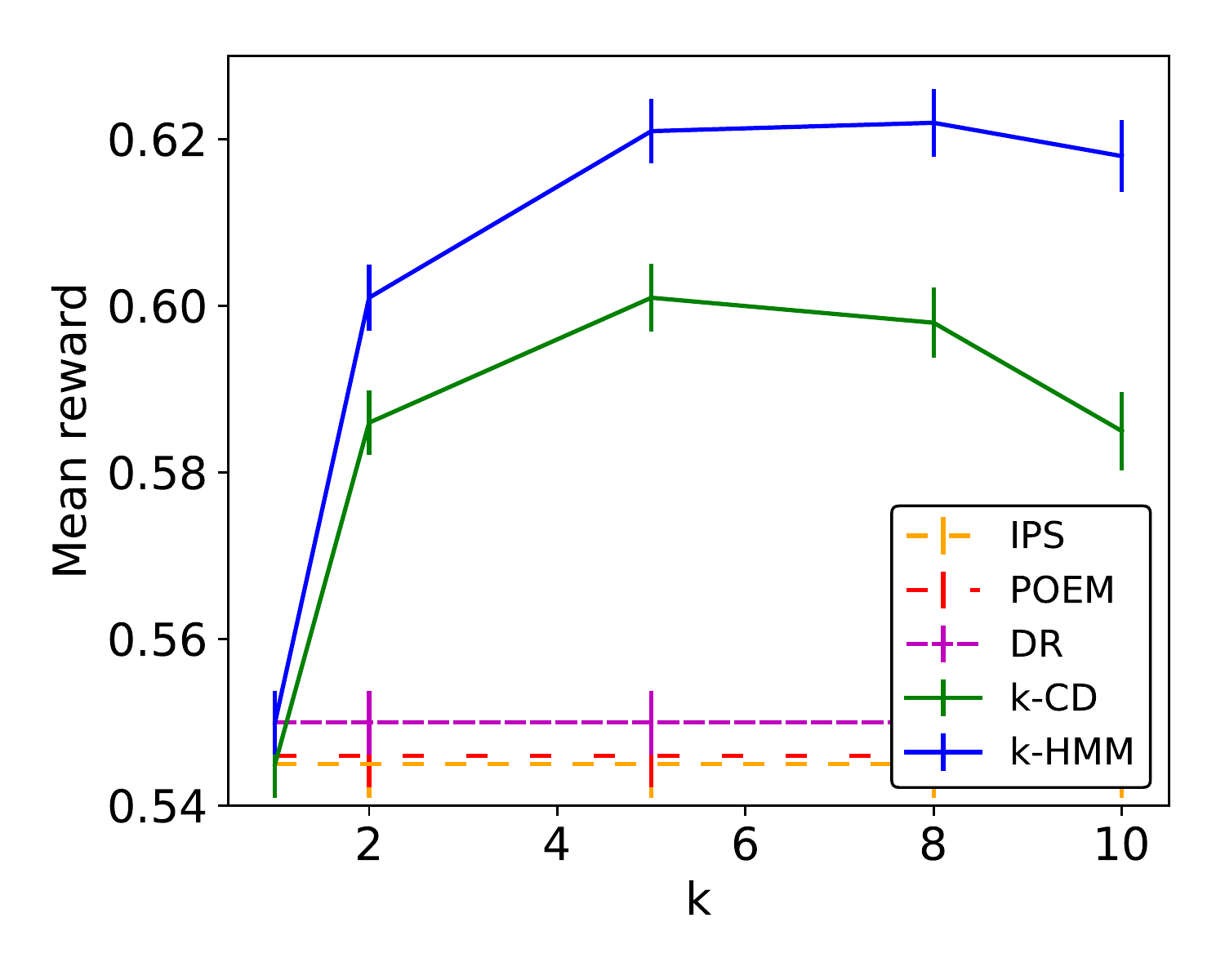}
\end{minipage}
\hspace{0.1in}\begin{minipage}{0.14\textwidth}
{\scriptsize
\begin{tabular}{@{} l r @{}}
    \toprule
    \textbf{Method} & Reward \\ 
    \midrule
    IPS & $0.545$ \\ 
    DR & $0.550$ \\
    POEM & $0.546$ \\ 
    \midrule
    {Ours:} \\
    $k$-CD & $0.601$ \\
    \textbf{$k$-HMM} & $\bm{0.621}$ \\
    \bottomrule
\end{tabular}}
\end{minipage}%
\caption{Mean rewards and their standard deviations in the synthetic dataset. The results are averaged over $10$ runs. The table shows results for $k = 5$.}
\label{fig:sim_plot}
\end{center}
\end{figure}

\begin{figure*}[t]
\begin{center}
\begin{minipage}{0.32\textwidth}
\includegraphics[width=\textwidth]{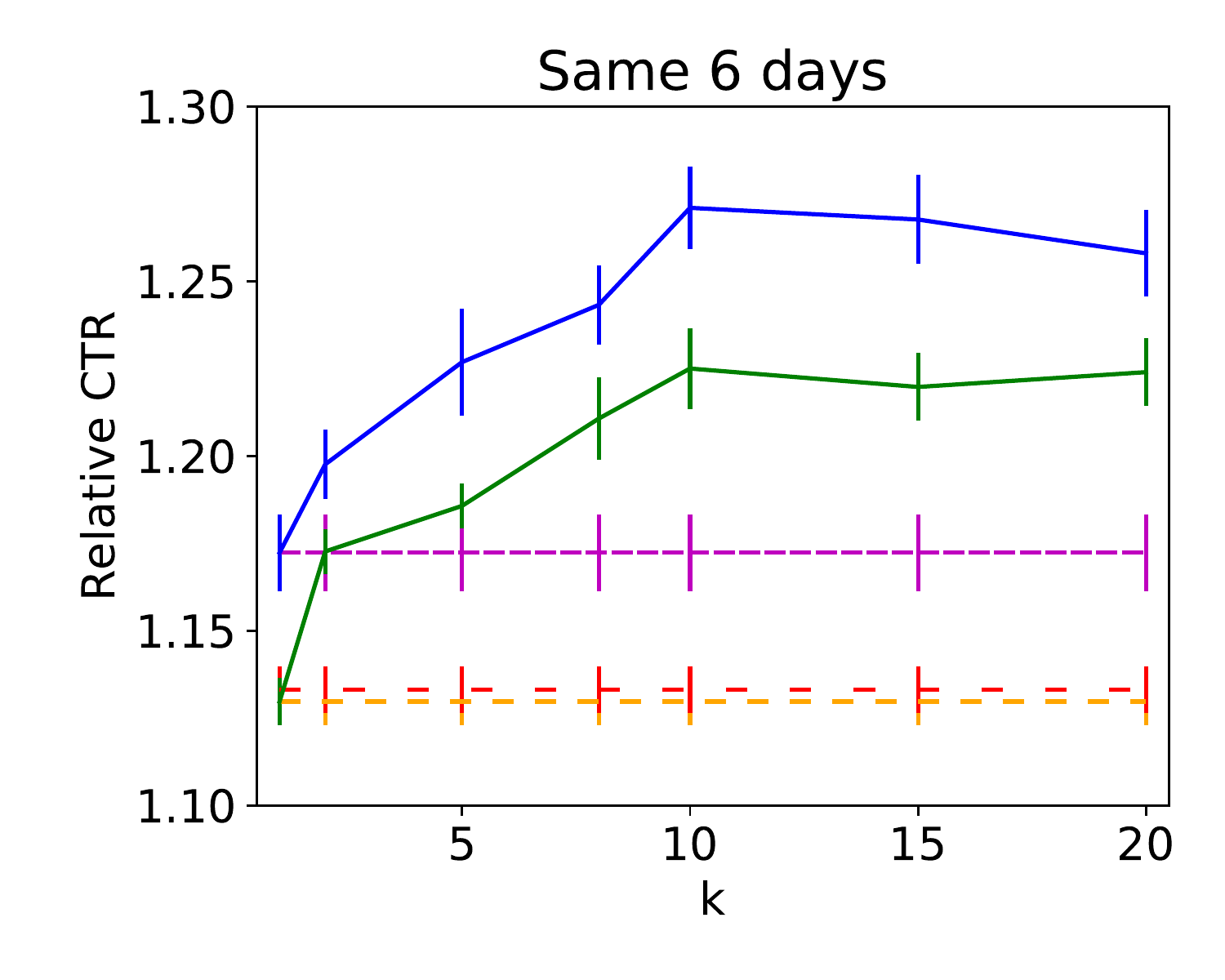}
\end{minipage}
\begin{minipage}{0.32\textwidth}
 \includegraphics[width=\textwidth]{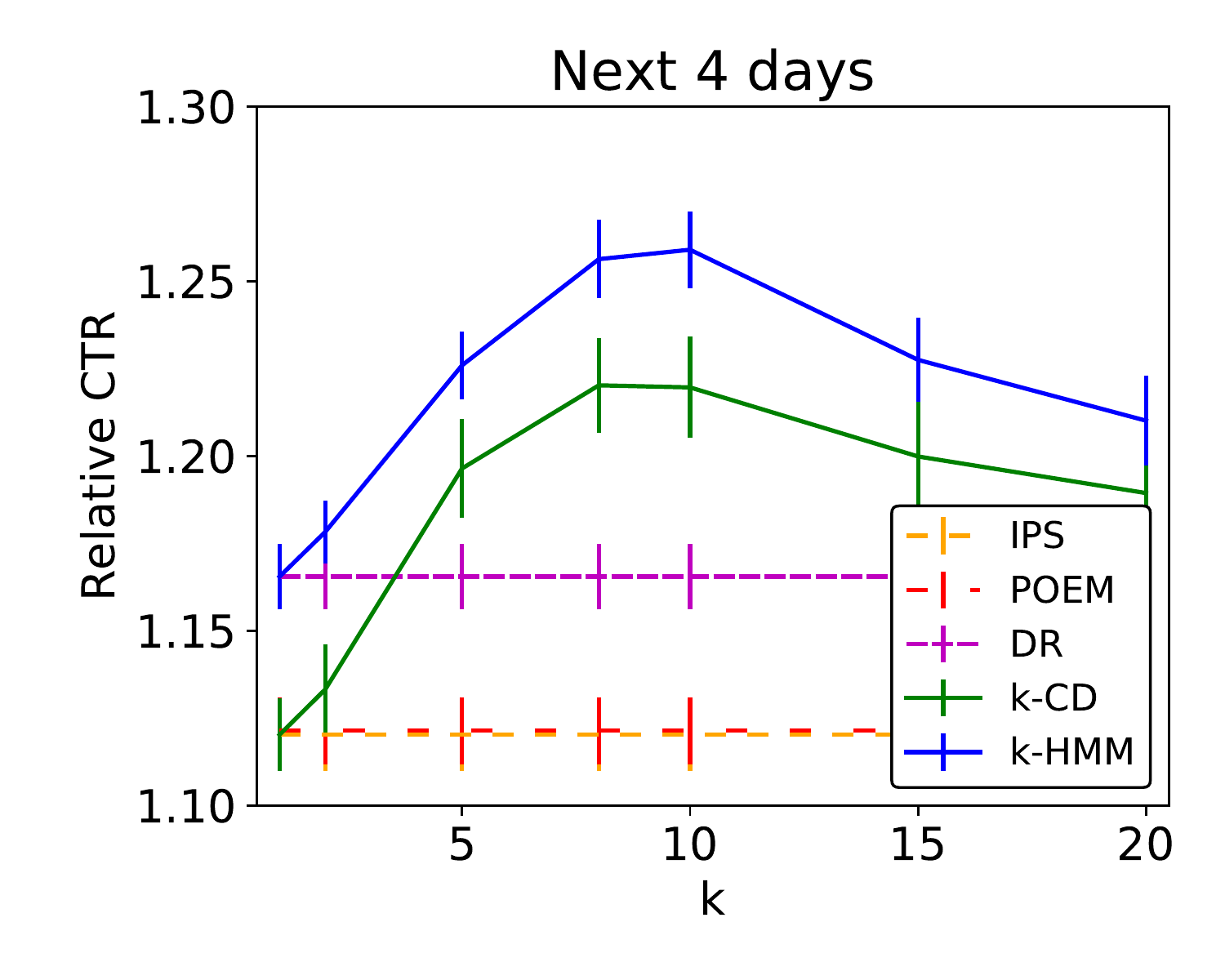}
\end{minipage}
\begin{minipage}{0.33\textwidth}
\hspace{0.1in}{\scriptsize
\begin{tabular}{@{} l r @{\ \  \ } r @{}}
    \toprule
    \textbf{Method} & same 6 days & next 4 days \\ 
    \midrule
    IPS & $1.13 \pm 0.006$ & $1.12 \pm 0.010$\\ 
    DR & $1.16 \pm 0.011$ & $1.17 \pm 0.009$\\
    POEM & $1.13 \pm 0.008$ & $1.13 \pm 0.009$ \\ 
    \midrule
    {Ours:} \\
    $k$-CD & $1.21 \pm 0.012$ & $1.21 \pm 0.010$\\
    \textbf{$k$-HMM} & $\bm{1.25 \pm 0.011}$ & $\bm{1.24 \pm 0.011}$ \\
    \bottomrule
\end{tabular}}
\end{minipage}%
\caption{Mean relative CTRs and their standard deviations in the Yahoo! dataset. The results are averaged over $10$ runs. The table shows results for $k = 10$.}
\label{fig:ydata_plot}
\end{center}
\end{figure*}

\subsection{Yahoo! Dataset}

We also experiment with the Yahoo! clickstream dataset \citep{yahoo_contextual}, which consists of real user interactions. In each interaction, a document is uniformly sampled from a pool of documents to show to a user, and whether the document is clicked by the user is logged. In prior work, the average click-through rate (CTR) of documents across users was empirically verified to change over time \citep{m_ucb,nonstationary_contextual_sigir}

We construct our logged dataset as follows. To reduce the size of the data, we choose a $6$-day horizon and randomly subsample one interaction per second over that horizon. For each sampled interaction, we choose a random subset of $10$ documents that could be shown to the user, to ensure the same number of actions in each round. The context for each interaction is a concatenation of the feature vectors of all $10$ sampled documents. The actions are documents and their rewards are indicators of being clicked in the original dataset. The result of this preprocessing is a logged dataset with horizon $T = 86,400 \times 6 = 518,400$ and $K = 10$ actions. It is important to note that the CTR for each document is likely to be non-stationary, and change smoothly. Hence, this experiment shows that our algorithms perform well even when our modeling assumptions may not hold.

We learn policies offline using the same methods as in \cref{sec:synthetic dataset}. Because our switching strategies depend on past interactions, offline evaluation of such policies from logged data is challenging. One approach is rejection sampling \citep{unbiased_evaluation}; but that can be sample inefficient. We remedy this by constructing a semi-synthetic piecewise-stationary bandit environment. In this environment, the CTR of a document in a given round is estimated from a half-day window around that round, and the click is sampled from a Bernoulli distribution with that mean. The half-day window is to model the non-stationarity of clicks.

We evaluate our learned policies in online deployment in two different bandit experiments. In the first experiment, we sub-sample interactions from the same $6$-day horizon, one per second. This approximately ensures that the underlying latent sequence is the same as in the logged data, which is the special case that we analyze. In the second experiment, we sub-sample interactions from the next $4$ days of data, which potentially have a dramatically different latent state sequence. In \cref{fig:ydata_plot}, we report relative CTRs for all compared methods, averaged over $10$ runs. We also plot the relative CTR of $k$-CD and $k$-HMM methods as a function of the estimated number of latent states, $k$. Both of our approaches perform the best, with $k$-HMM being better due to learning a full environment model. Our methods outperform stationary baselines by up to $10\%$. These results show that even in situations with a non-obvious latent state structure, our approach improves over methods that ignore latent states.

\section{Related Work}
\label{sec:related work}

We study off-policy learning in a non-stationary bandit setting. Both areas have been individually well-explored in prior literature.

\textbf{Non-stationary bandits.} The problem of non-stationary rewards is well-studied in bandit literature~\citep{nonstationary_mab,sliding_window_ucb}. First works adapted to changes passively by weighting rewards, either by exponential discounting \citep{discounted_ucb} or by considering recent rewards in a sliding window~\citep{sliding_window_ucb}. In the adversarial setting \citep{exp3_s,SHIFTBAND}, adaptation can be achieved by bounding the weights of experts from below. These algorithms have state-of-the-art switching regret, which we leverage in the online component of our algorithm. Recent works in piecewise-stationary bandits explored the idea of monitoring reward changes with a change-point detector. The detector examines differences in their distributions \citep{cusum} or empirical means ~\citep{m_ucb}. Such algorithms have state-of-the-art theoretical and empirical performance, and can be extended with similar guarantees to the contextual setting ~\citep{nonstationary_contextual_colt,nonstationary_contextual_sigir}.

\textbf{Off-policy learning.} Many works in off-policy learning have been devoted to building counterfactual estimators for evaluating policies. The unbiased IPS estimator has optimal theoretical guarantees when the logging policy is known or estimated well ~\citep{log_learning,surrogate_policy}. Various techniques have been employed to reduce the variance of IPS estimators, such as importance weight clipping ~\citep{clipping,msft_paper} or learning a reward model, to improve the MSE of the estimator ~\citep{dr,mrdr,switch,surrogate_objective}. Off-policy estimators can be directly applied to learning policies by optimizing the estimated value. Recent works in off-policy optimization additionally regularized the estimated value with its empirical standard deviation~\citep{poem} or used self-normalization as control variates~\citep{norm_poem}. Combinatorial actions, which are common in learning to rank, have been also explored \citep{slate_recommendation,ranking,topk_off_policy}.

Prior work in off-policy learning in non-stationary bandits is sparse, and has focused solely on evaluating a fixed target policy. Such works utilized time-series forecasting of future values \citep{arima_ope} or passively reweighed past observations ~\citep{jagerman}. There are also related works in offline evaluation of history-dependent policies in stationary environments ~\citep{unbiased_evaluation,eff_evaluation}. We are the first to provide a comprehensive method for both off-policy optimization and online policy selection in non-stationary environments.

\section{Conclusions}
\label{sec:conclusions}

In this work, we take first steps for off-policy optimization in non-stationary environments. Our algorithms partition the offline logged data by latent state, and optimize latent sub-policies conditioned on the partitions. We propose two techniques to partition the data: change-point detection and HMM. We prove high-probability bounds on the quality of off-policy optimized sub-policies and their regret during online deployment. Finally, we empirically validate our approach in synthetic and real-world data. We believe that our work is the first step in general off-policy optimization under non-stationarity. Our current approach uses simple non-stationary models of logged data. We propose using a change-point detector or HMM, but do not provide guarantees on HMMs due to lack of existing guarantees in inference. Good directions for future work are better models of non-stationarity, which could potentially handle smooth changes in the logged data.

% \todob{Make sure that the references are properly capitalized. "Discounted ucb" should be "Discounted UCB". "hidden markov models" should be "hidden Markov models". Check for others.}

\bibliographystyle{plainnat}
\bibliography{references}

\begin{thebibliography}{37}
\providecommand{\natexlab}[1]{#1}
\providecommand{\url}[1]{\texttt{#1}}
\expandafter\ifx\csname urlstyle\endcsname\relax
  \providecommand{\doi}[1]{doi: #1}\else
  \providecommand{\doi}{doi: \begingroup \urlstyle{rm}\Url}\fi

\bibitem[Abbasi-yadkori et~al.(2011)Abbasi-yadkori, P\'{a}l, and
  Szepesv\'{a}ri]{improved_linucb}
Yasin Abbasi-yadkori, D\'{a}vid P\'{a}l, and Csaba Szepesv\'{a}ri.
\newblock Improved algorithms for linear stochastic bandits.
\newblock \emph{NeurIPS}, 2011.

\bibitem[Auer(2002)]{SHIFTBAND}
Peter Auer.
\newblock Using confidence bounds for exploitation-exploration trade-offs.
\newblock \emph{Journal of Machine Learning Research}, 2002.

\bibitem[Auer et~al.(2002)Auer, Cesa-Bianchi, Freund, and Schapire]{exp3_s}
Peter Auer, Nicolò Cesa-Bianchi, Yoav Freund, and Robert~E Schapire.
\newblock The nonstochastic multiarmed bandit problem.
\newblock In \emph{SIAM journal on computing}, 2002.

\bibitem[Baum and Petrie(1966)]{hmm}
Leonard~E. Baum and Ted Petrie.
\newblock Statistical inference for probabilistic functions of finite state
  markov chains.
\newblock \emph{The Annals of Mathematical Statistics}, 1966.

\bibitem[Beshes et~al.(2014)Beshes, Gur, and Zeevi]{nonstationary_mab}
Omar Beshes, Yonatan Gur, and Assaf Zeevi.
\newblock Stochastic multi-armed-bandit problem with non-stationary rewards.
\newblock \emph{NIPS}, 2014.

\bibitem[Bottou et~al.(2013)Bottou, Peters, Quiñonero-Candela, Charles,
  Chickering, Portugaly, Ray, Simard, and Snelson]{msft_paper}
Léon Bottou, Jonas Peters, Joaquin Quiñonero-Candela, Denis~X Charles, D~Max
  Chickering, Elon Portugaly, Dipankar Ray, Patrice Simard, and Ed~Snelson.
\newblock Counterfactual reasoning and learning systems: The example of
  computational advertising.
\newblock \emph{The Journal of Machine Learning Research}, 2013.

\bibitem[Cao et~al.(2019)Cao, Wen, Kveton, and Xie]{m_ucb}
Yang Cao, Zheng Wen, Branislav Kveton, and Yao Xie.
\newblock Nearly optimal adaptive procedure with change detection for
  piecewise-stationary bandit.
\newblock \emph{AISTATS}, 2019.

\bibitem[Chen et~al.(2019{\natexlab{a}})Chen, Beutel, Covington, Jain,
  Belletti, and Chi]{topk_off_policy}
Minmin Chen, Alex Beutel, Paul Covington, Sagar Jain, Francois Belletti, and
  Ed~H. Chi.
\newblock Top-k off-policy correction for a {REINFORCE} recommender system.
\newblock \emph{WSDM}, 2019{\natexlab{a}}.

\bibitem[Chen et~al.(2019{\natexlab{b}})Chen, Gummadi, Harris, and
  Schuurmans]{surrogate_objective}
Minmin Chen, Ramki Gummadi, Chris Harris, and Dale Schuurmans.
\newblock Surrogate objectives for batch policy optimization in one-step
  decision making.
\newblock \emph{NIPS}, 2019{\natexlab{b}}.

\bibitem[Dudik et~al.(2011)Dudik, Langford, and Li]{dr}
Miroslav Dudik, John Langford, and Lihong Li.
\newblock Doubly robust policy evaluation and learning.
\newblock \emph{ICML}, 2011.

\bibitem[Dudik et~al.(2012)Dudik, Erhan, Langford, and Li]{eff_evaluation}
Miroslav Dudik, Dumitru Erhan, John Langford, and Lihong Li.
\newblock Sample-efficient nonstationary policy evaluation for contextual
  bandits.
\newblock \emph{UAI}, 2012.

\bibitem[Farajtabar et~al.(2018)Farajtabar, Chow, and Ghamvamzadeh]{mrdr}
Mehrdad Farajtabar, Yinlam Chow, and Mohammad Ghamvamzadeh.
\newblock More robust doubly robust off-policy evaluation.
\newblock \emph{ICML}, 2018.

\bibitem[Garivier and Moulines(2008)]{sliding_window_ucb}
Aurélien Garivier and Eric Moulines.
\newblock On upper-confidence bound policies for non-stationary bandit
  problems.
\newblock \emph{International Conference on Algorithmic Learning Theory}, 2008.

\bibitem[Hartland et~al.(2007)Hartland, Baskiotis, Gelly, Sebag, and
  Teytaud]{piecewse_stationary}
Cédric Hartland, Nicolas Baskiotis, Sylvain Gelly, Michèle Sebag, and Olivier
  Teytaud.
\newblock Change point detection and meta-bandits for online learning in
  dynamic environments.
\newblock \emph{CAp}, 2007.

\bibitem[Hong et~al.(2020)Hong, Kveton, Zaheer, Chow, Ahmed, and
  Boutilier]{latent_bandits_revisited}
Joey Hong, Branislav Kveton, Manzil Zaheer, Yinlam Chow, Amr Ahmed, and Craig
  Boutilier.
\newblock Latent bandits revisited.
\newblock In \emph{NeurIPS}, 2020.

\bibitem[Horvitz and Thompson(1952)]{ips}
D.~G. Horvitz and D.~J. Thompson.
\newblock A generalization of sampling without replacement from a finite
  universe.
\newblock \emph{Journal of the American Statistical Association}, 1952.

\bibitem[Hsu et~al.(2008)Hsu, Kakade, and Zhang]{spectral_hmm}
Daniel~J. Hsu, Sham~M. Kakade, and Tong Zhang.
\newblock A spectral algorithm for learning hidden markov models.
\newblock \emph{CoRR}, abs/0811.4413, 2008.

\bibitem[Ionides(2008)]{clipping}
Edward~L Ionides.
\newblock Truncated importance sampling.
\newblock \emph{Journal of Computational and Graphical Statistics}, 2008.

\bibitem[Jagerman et~al.(2019)Jagerman, Markov, and de~Rijke]{jagerman}
Rolf Jagerman, Ilya Markov, and Maarten de~Rijke.
\newblock When people change their mind: Off-policy evaluation in
  non-stationary recommendation environments.
\newblock \emph{WSDM}, 2019.

\bibitem[Kocsis and Szepesvári(2006)]{discounted_ucb}
Levente Kocsis and Csaba Szepesvári.
\newblock Discounted ucb.
\newblock \emph{In 2nd PASCAL Challenges Workshop}, 2006.

\bibitem[Langford and Zhang(2008)]{greedy_side_info}
John Langford and Tong Zhang.
\newblock The epoch-greedy algorithm for multi-armed bandits with side
  information.
\newblock \emph{NeurIPS}, 2008.

\bibitem[Lattimore and Szepesvári(2019)]{bandit_book}
Tor Lattimore and Csaba Szepesvári.
\newblock \emph{Bandit Algorithms}.
\newblock Cambridge University Press, 2019.
\newblock \doi{10.1017/9781108571401}.

\bibitem[Li et~al.(2010)Li, Chu, Langford, and Schapire]{yahoo_contextual}
Lihong Li, Wei Chu, John Langford, and Robert~E. Schapire.
\newblock A contextual-bandit approach to personalized news article
  recommendation.
\newblock \emph{WWW}, 2010.

\bibitem[Li et~al.(2011)Li, Chu, Langford, and Wang]{unbiased_evaluation}
Lihong Li, Wei Chu, John Langford, and Xuanhui Wang.
\newblock Unbiased offline evaluation of contextual bandit-based news article
  recommendation algorithms.
\newblock \emph{WSDM}, 2011.

\bibitem[Li et~al.(2018)Li, Abbasi{-}Yadkori, Kveton, Muthukrishnan, Vinay, and
  Wen]{ranking}
Shuai Li, Yasin Abbasi{-}Yadkori, Branislav Kveton, S.~Muthukrishnan, Vishwa
  Vinay, and Zheng Wen.
\newblock Offline evaluation of ranking policies with click models.
\newblock \emph{KDD}, 2018.

\bibitem[Liu et~al.(2018)Liu, Lee, and Shroff]{cusum}
Fang Liu, Joohyun Lee, and Ness~B. Shroff.
\newblock A change-detection based framework for piecewise-stationary
  multi-armed bandit problem.
\newblock \emph{AAAI}, 2018.

\bibitem[Luo et~al.(2018)Luo, Agarwal, and
  Langford]{nonstationary_contextual_colt}
Haipeng Luo, Alekh Agarwal, and John Langford.
\newblock Efficient contextual bandits in non-stationary worlds.
\newblock \emph{COLT}, 2018.

\bibitem[Rosman et~al.(2016)Rosman, Hawasly, and Ramamoorthy]{bpr}
Benjamin Rosman, Majd Hawasly, and Subramanian Ramamoorthy.
\newblock Bayesian policy reuse.
\newblock \emph{Machine Learning}, 2016.

\bibitem[Strehl et~al.(2010)Strehl, Langford, Li, and Kakade]{log_learning}
Alexander~L. Strehl, John Langford, Lihong Li, and Sham~M. Kakade.
\newblock Learning from logged implicit exploration data.
\newblock \emph{NIPS}, 2010.

\bibitem[Swaminathan and Joachims(2015{\natexlab{a}})]{norm_poem}
Adith Swaminathan and Thorsten Joachims.
\newblock The self-normalized estimator for counterfactual learning.
\newblock \emph{NIPS}, 2015{\natexlab{a}}.

\bibitem[Swaminathan and Joachims(2015{\natexlab{b}})]{poem}
Adith Swaminathan and Thorsten Joachims.
\newblock Counterfactual risk minimization: Learning from logged bandit
  feedback.
\newblock \emph{ICML}, 2015{\natexlab{b}}.

\bibitem[Swaminathan et~al.(2016)Swaminathan, Krishnamurthy, Agarwal, Dudik,
  Langford, Jose, and Zitouni]{slate_recommendation}
Adith Swaminathan, Akshay Krishnamurthy, Alekh Agarwal, Miroslav Dudik, John
  Langford, Damien Jose, and Imed Zitouni.
\newblock Off-policy evaluation for slate recommendation.
\newblock \emph{NIPS}, 2016.

\bibitem[Thomas et~al.(2017)Thomas, Theocharous, Ghavamzadeh, Durugkar, and
  Brunskill]{arima_ope}
Philip~S. Thomas, Georgios Theocharous, Mohammad Ghavamzadeh, Ishan Durugkar,
  and Emma Brunskill.
\newblock Predictive off-policy policy evaluation for nonstationary decision
  problems, with applications to digital marketing.
\newblock \emph{AAAI}, 2017.

\bibitem[Wang et~al.(2017)Wang, Agarwal, and Dudik]{switch}
Yu-Xiang Wang, Alekh Agarwal, and Miroslav Dudik.
\newblock Optimal and adaptive off-policy evaluation in contextual bandits.
\newblock \emph{ICML}, 2017.

\bibitem[Wu et~al.(2018)Wu, Iyer, and Wang]{nonstationary_contextual_sigir}
Qingyun Wu, Naveen Iyer, and Hongning Wang.
\newblock Learning contextual bandits in a non-stationary environment.
\newblock \emph{SIGIR}, 2018.

\bibitem[Xie et~al.(2019)Xie, Liu, Liu, Wang, Zhou, and Peng]{surrogate_policy}
Yuan Xie, Boyi Liu, Qiang Liu, Zhaoran Wang, Yuan Zhou, and Jian Peng.
\newblock Off-policy evaluation and learning from logged bandit feedback: Error
  reduction via surrogate policy.
\newblock \emph{ICLR}, 2019.

\bibitem[Yu and Mannor(2009)]{wmd}
Jia~Yuan Yu and Shie Mannor.
\newblock Piecewise-stationary bandit problems with side observations.
\newblock In \emph{International Conference on Machine Learning}, 2009.

\end{thebibliography}

\clearpage
\onecolumn
\appendix

\section{Proofs for Offline Policy Optimization}
\label{sec:proof_offline}

Recall that we have a fixed latent sequence $z_{1:T}$ such that for round $t$, latent state $z_t$ parameterizes the underlying distribution of reward $r_t \in [0,1]$. Also recall that we have IPS estimator $\hatV$ given in \eqref{eqn:z_ips}, where the clipping parameter $M$ can be ignored by only considering policies in $\cH$. In this section, we denote by $\tilde{V}$ the IPS estimator in \eqref{eqn:z_ips} with the true latent states $z_{1 : T}$. By \cref{lem:unbiased_z_ips}, we know that $\tilde{V}$ is unbiased.

Our first result bounds the discrepancy between the two IPS estimators $\tilde{V}(\Pi)$ and $\hat{V}(\Pi)$:

\begin{lemma} For any $\Pi \in \cH^\mathcal{Z}$ and $\delta \in (0, 1]$,
$
  \left|\hat{V}(\Pi) - \tilde{V}(\Pi)\right|
  \leq M \varepsilon(T, \delta)
$
holds with probability at least $1 - \delta$.
\label{lem:epsilon_oracle_error}
\end{lemma}
\begin{proof}
The claim is proved as
\begin{align*}
  \left|\hat{V}(\Pi) - \tilde{V}(\Pi)\right| 
  = \left|\sum_{t = 1}^T \frac{\pi_{\hat{z}_t}(a_t \mid x_t)}{p_t} r_t -
  \frac{\pi_{z_t}(a_t \mid x_t)}{p_t} r_t\right|
  \leq M \sum_{t = 1}^T \mathbbm{1}[\hat{z}_t \neq z_t]
  \leq M \varepsilon(T, \delta)\,.
\end{align*}
The first inequality is by assuming that $\cH$ in $\cH^\mathcal{Z}$ satisfy \eqref{eq:clipped policies}. The second inequality is by \cref{ass:oracle} in \cref{sec:evaluation} and holds with probability at least $1 - \delta$.
\end{proof}

Next, we bound the estimation error of $\tilde{V}(\Pi)$ from $V(\Pi)$. This error is due to the randomness in $\cD$.

\begin{lemma}
For any $\Pi \in \cH^\mathcal{Z}$, logged data $\cD$, and $\delta \in (0, 1]$,
$
  \left|\tilde{V}(\Pi) - V(\Pi)\right|
  \leq M \sqrt{2T \log(2 / \delta)}
$
holds with probability at least $1 - \delta$.
\label{lem:data_randomness_error}
\end{lemma}
\begin{proof}
We define a martingale sequence $(U_t)_{t \in [T] \cup \{0\}}$ over rounds $t$ and then use Azuma's inequality. The sequence is defined as $U_0 = 0$ and
\begin{align*}
  U_t
  = U_{t - 1} + \frac{\pi_{z_t}(a_t \mid x_t)}{p_t} r_t - V_t(\pi_{z_t})
\end{align*}
for $t > 0$. It is easy to verify that this is a martingale. In particular, since $z_t$ is fixed,
\begin{align*}
  \E{x_t, a_t, r_t \sim P_{z_t}, \pi_0}
  {\frac{\pi_{z_t}(a_t \mid x_t)}{p_t} r_t - V_t(\pi_{z_t}) \, \middle| \, U_0, \dots, U_{t - 1}}
  = \E{x_t, a_t, r_t \sim P_{z_t}, \pi_{z_t}}{r_t} - V_t(\pi_{z_t})
  = 0\,,
\end{align*}
and $\E{}{U_t \mid U_0, \dots, U_{t - 1}} = U_{t - 1}$ for any round $t$. Also, since $\Pi \in \cH^{\mathcal{Z}}$, we have
\begin{align*}
  \left|\frac{\pi_{z_t}(a_t \mid x_t)}{p_t} r_t - V_t(\pi_{z_t})\right|
  \leq M\,.
\end{align*}
Finally, by Azuma's inequality, we get
\begin{align*}
  \prob{|\tilde{V}(\Pi) - V(\Pi)| \geq M \sqrt{2T \log(2 / \delta)}}
  = \prob{|U_T - U_0| \geq M \sqrt{2T \log(2 / \delta)}}
  \leq 2 \exp\left[- \frac{4 M^2 T \log(2 / \delta)}{2 M^2 T}\right]
  \leq \delta\,.
\end{align*}
This concludes the proof.
\end{proof}

Using \cref{lem:epsilon_oracle_error,lem:data_randomness_error} above, we can derive the results stated in the main paper.

\evalmain*
\begin{proof}
We have
\begin{align*}
  \left| \hat{V}(\Pi) - V(\Pi) \right|
  \leq \left| \hat{V}(\Pi) - \tilde{V}(\Pi) \right| + \left| \tilde{V}(\Pi) - V(\Pi) \right|
\end{align*}
from the triangle inequality. The result follows from \cref{lem:epsilon_oracle_error} and \cref{lem:data_randomness_error}. 
\end{proof}

\optmain*
\begin{proof}
We have
\begin{align*}
  V(\Pi^*) - V(\hat{\Pi}) 
  = \left[ V(\Pi^{*}) -  \hat{V}(\hat{\Pi}) \right] +
  \left[ \hat{V}(\hat{\Pi}) - V(\hat{\Pi}) \right]
  \leq \left[ V(\Pi^*) - \hat{V}(\Pi^*) \right] +
  \left[ \hat{V}(\hat{\Pi}) - V(\hat{\Pi}) \right]\,,
\end{align*}
where the inequality is from $\hat{\Pi}$ maximizing $\hat{V}$. By \cref{lem:eval_main}, we have for any $\Pi \in \cH^{\mathcal{Z}}$ that
\begin{align*}
  |\hat{V}(\Pi) - V(\Pi)|
  \leq M \varepsilon(T, \delta_1/2) + 2M \sqrt{T \log(4/\delta_2)}
\end{align*}
holds with probability at least $1 - \delta_1/2 - \delta_2/2$. We apply the lemma to both $\hat{\Pi}$ and $\Pi^*$, and get the desired result.
\end{proof}

\section{Proofs for Change-Point Detector}
\label{sec:proof_cd}

Recall that $S$ is the number of stationary segments, and $\tau_0 = 1 < \tau_1 < \hdots < \tau_{S-1} < T = \tau_{S}$ are the change-points. Also recall that we have change-point detector given by \cref{alg:oracle} that on a high-level, computes differences in total reward across sliding windows of length $w$ and detects a change-point if a difference exceeds threshold $c$.
For any $i \in [S - 1]$, let $W_i = [\tau_i - w, \tau_i + w]$ be $w$-close rounds to change-point $\tau_i$. We also define $W = \bigcup_i W_i$ as all rounds $w$-close to any change-point.

First, we bound the probability of false positives, or that we declare any round $t \not\in W$ as a change-point:

\begin{lemma}For any round $t \not\in W$, the probability of a false detection is bounded from above as
\begin{align*}
  \prob{\abs{\mu_t^- - \mu_t^+} \geq c}
  \leq 4 \exp\left[- \frac{w c^2}{2}\right]\,.
\end{align*}
\label{lem:cd_false_positive}
\end{lemma}
\begin{proof}
Since $t \not\in \bigcup_i W_i$, we have $\E{}{\mu_t^-} = \E{}{\mu_t^+}$. By Hoeffding's inequality, we get
\begin{align*}
  \prob{\abs{\mu_t^- - \mu_t^+} \geq c}
  \leq \prob{\abs{\mu_t^- - \E{}{\mu_t^-}} \geq c / 2} +
  \prob{\abs{\mu_t^+ - \E{}{\mu_t^+}} \geq c / 2}
  \leq \exp\left[- \frac{w c^2}{2}\right]\,.
\end{align*}
This concludes the proof.
\end{proof}

Next we bound the probability of failing to detect a change-point in $W$:

\begin{lemma} For any positive $c \leq \Delta / 2$ and $W_i$, a change-point is not detected in $W_i$ with probability at most
\begin{align*}
  \prob{\forall t \in W_i: \abs{\mu_t^- - \mu_t^+} \leq c}
  \leq 4 \exp\left[- \frac{w c^2}{2}\right]\,.
\end{align*}
\label{lem:cd_recall}
\end{lemma}
\begin{proof}
Fix $s = \tau_i$. From $s \in W_i$, we have
\begin{align*}
  \prob{\forall t \in W_i: \abs{\mu_t^- - \mu_t^+} \leq c}
  & = 1 - \prob{\exists t \in W_i: \abs{\mu_t^- - \mu_t^+} > c}
  \leq 1 - \prob{\abs{\mu_s^- - \mu_s^+} > c} \\
  & = \prob{\abs{\mu_s^- - \mu_s^+} \leq c}\,.
\end{align*}
Note that $\abs{\mu_s^- - \mu_s^+} \leq c$ implies that either $\mu_s^-$ or $\mu_s^+$ is not close to its mean. More specifically, since $\E{}{\mu_s^-} = V_{s - 1}(\pi_0)$, $\E{}{\mu_s^+} = V_s(\pi_0)$, and $\abs{V_s(\pi_0) - V_{s - 1}(\pi_0)} \geq \Delta$, we have
\begin{align*}
  \prob{\abs{\mu_s^- - \mu_s^+} \leq c}
  \leq \prob{\abs{\mu_s^- - \E{}{\mu_s^-}} \geq \frac{\Delta - c}{2}} +
  \prob{\abs{\mu_s^+ - \E{}{\mu_s^+}} \geq \frac{\Delta - c}{2}}\,.
\end{align*}
From $2 c \leq \Delta$ and by Hoeffding's inequality, the first term is bounded as
\begin{align*}
  \prob{\abs{\mu_s^- - \E{}{\mu_s^-}} \geq \frac{\Delta - c}{2}}
  \leq \prob{\abs{\mu_s^- - \E{}{\mu_s^-}} \geq c / 2}
  \leq 2 \exp\left[- \frac{w c^2}{2}\right]\,.
\end{align*}
The second term is bounded analogously. Finally, we chain all inequalities and get our claim.
\end{proof}

Finally, we prove \cref{thm:cd_oracle} by applying \cref{lem:cd_false_positive} to all rounds $t \not\in W$, \cref{lem:cd_recall} to all change-points, and then chaining them by the union bound.

\cdoracle*
\begin{proof}
Define $\delta \in (0, 1].$ We see that given $w$, setting $c$ as described satisfies,
$$ 4T\exp\left[\frac{-wc^2}{2}\right], \quad
4k\exp\left[\frac{-wc^2}{2}\right] \leq \frac{\delta}{2}.$$
We know that $\varepsilon(T, \delta) = kw$ when all the estimated changepoints are in $W$ (at most $w$ rounds from a true change-point), and every $W_i \in W$ contains exactly one estimated change-point. This cannot happen if (1) a change-point is falsely detected outside $W$, and (2), no change-point is detected in some $W_i \in W$.

We can bound from above the probability of any error occurring with the union bound. Proposition 3 applied to every round upper-bounds the probability of (1) by $4T\exp\left(-wc^2/2\right)$. Meanwhile, Proposition 4 applied to every change-point upper-bounds the probability of (2) by $4k\exp\left(-wc^2/2\right)$. From Algorithm~\ref{alg:oracle}, we remove a $4w$-window around each detected changepoint, and under the assumption that $\tau_i - \tau_{i - 1} > 4w$ for all $i \in [k]$, we guarantee that exactly one changepoint is detected in each $W_i$ for true changepoint $\tau_i$. Combining yields the total probability of an error,
$$ 4T\exp\left[\frac{-wc^2}{2}\right] +
4k\exp\left[\frac{-wc^2}{2}\right] \leq \delta,$$
which is the desired result.
\end{proof}

\section{Proofs for Online Deployment}
\label{sec:proof_online}

Recall that we have a mixture-of-experts algorithm $\mathcal{E}$ and experts/sub-policies $\hat{\Pi} = (\hat{\pi})_{z \in \mathcal{Z}}$, such that for each round $t$, actions are sampled according to $a_t \sim \mathcal{E}_t(x_t, \hat{\pi})$. Let $\mathcal{E}$ be Exp4.S as described in \cref{alg:exp4_s}; this is similar to one proposed in \citet{nonstationary_contextual_colt}, but for stochastic experts. 

\begin{algorithm}[H]
\caption{Exp4.S}\label{alg:exp4_s}
\KwIn{vector of expert sub-policies $\hat{\Pi} = (\hat{\pi}_z)_{z \in \mathcal{Z}}$ with $|\mathcal{Z}| = L$, and hyperparameters
$\beta, \eta > 0, \gamma \in (0, 1]$
}
\BlankLine
Initialize $w_1 = (1/L, \hdots, 1/L) \in [0, 1]^L$. \\
\For {$t \gets 1, 2, \hdots, T$}{%
    Observe $x_t$ and expert feedback $\hat{\pi}_z(\cdot \mid x_t), \, \forall z \in \mathcal{Z}$. \\

    Choose $a_t \sim \mathcal{E}_t$, where for each $a \in \cA$,
    $$\mathcal{E}_t(a) = (1 - \gamma) \sum_{z \in \mathcal{Z}} w_t(z) \hat{\pi}_z(a \mid x_t) 
    + \frac{\gamma}{L}\,.$$
    
    Observe $r_t$
    
    Estimate the action costs under full feedback $\hat{c}_t(a) = \mathbbm{1}[a_t = a] \frac{1 - r_t}{\mathcal{E}_t(a)}$, $\forall a \in \cA$.

    Propagate the cost to the experts $\tilde{c}_t(z) = \hat{c}_t(a_t) \hat{\pi}_z(a_t \mid x_t)$, $\forall z \in \mathcal{Z}$. 
    
    Update the distribution weights, 
    $\tilde{w}_{t+1}(z) \propto  w_t(z) \exp\left(-\eta \tilde{c}_t(z)\right)$, $\forall z \in \mathcal{Z}$.
    
    Mix with uniform weights, 
    $w_{t+1}(z) = (1 - \beta) w_t(z) + \beta$, $\forall z \in \mathcal{Z}$.
}
\end{algorithm}

Our first result is the following regret guarantee over any stationary segment. A version of this proof for deterministic experts is in Theorem 2 of \citet{nonstationary_contextual_colt}.
\begin{lemma} 
\label{lem:interval_exp4}
Let $\mathcal{E}$ be Exp4.S as in \cref{alg:exp4_s}. Also, let $\gamma = 0, \eta = \sqrt{\log(L) / (\ell K)}$, and $\beta = 1/L$. Then, for any stationary segment $[\tau_{s - 1}, \tau_s - 1]$ of length at most $\ell$, any history up to $\tau_{s - 1}$, and any latent state $z \in \mathcal{Z}$, the regret is bounded as
\begin{align*}
    \sum_{t = \tau_{s - 1}}^{\tau_s - 1}\E{z_t, \hat{\pi}_z}{r_t} - \E{z_t, \mathcal{E}_t}{r_t}
    \leq \sqrt{2 \ell K \log(L)}\,.
\end{align*}
\end{lemma}
\begin{proof}
First, we have the following upper-bound,
\begin{align*}
    \log \left[\sum_{z' \in \mathcal{Z}} w_t(z') \exp(-\eta \tilde{c}_t(z'))\right]
    &\leq \log \left[\sum_{z' \in \mathcal{Z}} w_t(z') \left(1 - \eta \tilde{c}_t(z') + \eta^2 \tilde{c}_t(z')^2\right)\right] \\
    &\leq -\eta \sum_{z' \in \mathcal{Z}} w_t(z') \tilde{c}_t(z') + \eta^2\sum_{z' \in \mathcal{Z}}w_t(z') \tilde{c}_t(z')^2\,,
\end{align*}
where we use that $\exp(-x) \leq 1 - x + x^2$, and $\log(1 + x) \leq x$ for all $x \geq 0$. 
Meanwhile, for any $z \in \mathcal{Z}$, we can also bound the same quantity from below,
\begin{align*}
    \log \left[\sum_{z' \in \mathcal{Z}} w_t(z') \exp(-\eta \tilde{c}_t(z'))\right]
    = \log \left[\frac{w_t(z) \exp(-\eta \tilde{c}_t(z))}{\tilde{w}_{t + 1}(z)}\right]
    &= \log \left[\frac{w_t(z)(1 - \beta)}{w_{t + 1}(z) - \beta}\right] -\eta \tilde{c}_t(z)  \\
    &\geq \log \left[\frac{w_t(z)}{w_{t + 1}(z)} \right] - 2\beta -\eta \tilde{c}_t(z)\,,
\end{align*}
where for the last inequality, we use that $\log(1 - \beta) \geq -\beta/(1 - \beta) \geq -2\beta$. Combining the two inequalities, summing over all $t \in [\tau_{s - 1}, \tau_s - 1]$, and telescoping yields,
\begin{align*}
    \sum_{t = \tau_{s - 1}}^{\tau_s - 1} \sum_{z' \in \mathcal{Z}} w_t(z') \tilde{c}_t(z') - \tilde{c}_t(z) &\leq \frac{1}{\eta}\log\left[\frac{w_{\tau_s}(z)}{w_{\tau_{s - 1}}(z)}\right]
    + \frac{2\beta \ell}{\eta} + \eta\sum_{t = \tau_{s - 1}}^{\tau_s - 1}\sum_{z' \in \mathcal{Z}}w_t(z') \tilde{c}_t(z')^2 \\
    &\leq 
    \frac{\log(1/ \beta) + 2\beta \ell}{\eta} 
    + \eta\sum_{t = \tau_{s - 1}}^{\tau_s - 1}\sum_{z' \in \mathcal{Z}}w_t(z') \tilde{c}_t(z')^2\,,
\end{align*}
where we use that $w_t(z) \in [\beta, 1]$ for all rounds $t$.

When $\gamma = 0$ we know that $\hat{c_t}(a_t)$ is unbiased, or $\E{z_t, \mathcal{E}_t}{\hat{c_t}(a_t)} = 1 - \E{z_t, \mathcal{E}_t}{r_t}$. We also have that for any $z'\in \mathcal{Z}$,
\begin{align*}
\E{z_t, \mathcal{E}_t}{\tilde{c}_t(z')} = \E{z_t, \mathcal{E}_t}{\sum_{a \in \cA} \hat{\pi}_{z'}(a \mid x_t) \hat{c}_t(a)} = 1 - \E{z_t, \hat{\pi}_z}{r_t}\,.
\end{align*}
Taking the expectation of both sides leads to,
\begin{align*}
    \sum_{t = \tau_{s - 1}}^{\tau_s - 1}\E{z_t, \hat{\pi}_z}{r_t} - \E{z_t, \mathcal{E}_t}{r_t} 
    \leq \frac{\log(1/ \beta) + 2\beta \ell}{\eta} 
    + \eta\sum_{t = \tau_{s - 1}}^{\tau_s - 1} \sum_{z' \in \mathcal{Z}}\E{z_t, \mathcal{E}_t}{w_t(z') \tilde{c}_t(z')^2}\,.
\end{align*}
Next, we have that for any $z' \in \mathcal{Z}$,
\begin{align*}
    \E{z_t, \mathcal{E}_t}{\tilde{c}_t(z')^2} = \E{z_t, \mathcal{E}_t}{\left(\frac{\hat{\pi}_{z'}(a_t \mid x_t) (1 - r_t)}{\mathcal{E}_t(a_t)}\right)^2}  \leq \sum_{a \in \cA} \frac{\hat{\pi}_{z'}(a \mid x_t)}{\mathcal{E}_t(a)}\,,
\end{align*}
where we use that $a_t \sim \mathcal{E}_t$ and $r_t \in [0, 1]$. Substituting this result yields,
\begin{align*}
    \sum_{z' \in \mathcal{Z}} \E{z_t, \mathcal{E}_t}{w_t(z') \tilde{c}_t(z')^2}
    &\leq \sum_{a \in \cA} \E{z_t, \mathcal{E}_t}{\frac{1}{\mathcal{E}_t(a)} \sum_{z' \in \mathcal{Z}} w_t(z') \pi_{z'}(a_t \mid x_t)} \leq K\,,
\end{align*}
where we again use that $a_t \sim \mathcal{E}_t$. Substituting into the regret bound and using the values for $\eta, \beta$ yields
\begin{align*}
    \sum_{t = \tau_{s - 1}}^{\tau_s - 1}\E{z_t, \hat{\pi}_z}{r_t} - \E{z_t, \mathcal{E}_t}{r_t}
    \leq \frac{\log(1/ \beta) + 2\beta \ell}{\eta} 
    + \eta K \ell
    &\leq \sqrt{2 \ell K \log(L)}\,,
\end{align*}
as desired.
\end{proof}

In practice, we do not know the lengths of stationary segments, and may not be able to find a tight upper-bound $\ell$ on the lengths of stationary segments. However, in our analysis, we can further partition stationary segments so that they do not exceed length $\ell$ at the cost of increasing the number of change-points. This is formalized in the following corollary.
\begin{lemma}
\label{lem:exp4}
Let $\mathcal{E}$ be Exp4.S as in \cref{alg:exp4_s}. Also, let $\gamma = 0, \eta = \sqrt{\log(L) / (\ell K)}$, and $\beta = 1/L$. Then, the total regret is bounded by
\begin{align*}
\sum_{s = 1}^S  \max_{z \in \mathcal{Z}} \sum_{t = \tau_{s-1}}^{\tau_s - 1} \E{z_t, \hat{\pi}_z}{r_t}
- \sum_{t = 1}^T \E{z_t, \mathcal{E}_t}{r_t}
\leq \left(T/\sqrt{\ell} + S \sqrt{\ell}\right)\sqrt{2K\log(L)}\,.
\end{align*}
\end{lemma}
\begin{proof}
Recall that $S$ is the number of stationary segments within the $T$ rounds, as defined in \cref{sec:setting}. Our goal is to divide the $T$ rounds into stationary intervals of length at most $\ell$, so that we can apply \cref{lem:interval_exp4} on each interval.
We do this as follows. First, we construct $T/\ell$ intervals of length at most $T$. Then, we additionally divide intervals that contain changepoints, so that each interval contains only a single latent state. This leads to at most $T/\ell + S$ stationary intervals. 
Finally, using \cref{lem:interval_exp4} on each interval and summing the regrets the desired result. Note that though we consider $T/\ell + S$ intervals, we only need to consider the best latent sub-policy for each of $S$ stationary segments, as intervals belonging to the same stationary segment have the same optimal sub-policy.
\end{proof}

\regretdecomposition*
\begin{proof}
The regret can be decomposed as follows:
\begin{align*}
    \mathcal{R}(T; \mathcal{E}, \hat{\Pi}) 
    &= \sum_{t = 1}^T \E{z_t, \pi^*_{z_t}}{r_t} - \sum_{t = 1}^T \E{z_t, \mathcal{E}_t}{r_t} \\
    &= \left[\sum_{t = 1}^T \E{z_t, \pi^*_{z_t}}{r_t}
    - \sum_{t = 1}^T \E{z_t, \hat{\pi}_{z_t}}{r_t}\right]
    + \left[\sum_{t = 1}^T \E{z_t, \hat{\pi}_{z_t}}{r_t}
    - \sum_{t = 1}^T \E{z_t, \mathcal{E}_t}{r_t}\right]\,,
\intertext{where we introduce $\hat{\Pi}$ that acts according to the true latent state. Then, recalling there are $S$ stationary segments, the above expression can be further expressed as}
    & \left[\sum_{t = 1}^T \E{z_t, \pi^*_{z_t}}{r_t}
    - \sum_{t = 1}^T \E{z_t, \hat{\pi}_{z_t}}{r_t}\right]
    + \left[\sum_{s = 1}^S \sum_{t = \tau_{s-1}}^{\tau_s - 1} \E{z_t, \hat{\pi}_{z_t}}{r_t}
    - \sum_{t = 1}^T \E{z_t, \mathcal{E}_t}{r_t}\right] \\
    & \quad \leq \left[\sum_{t = 1}^T \E{z_t, \pi^*_{z_t}}{r_t}
    - \sum_{t = 1}^T \E{z_t, \hat{\pi}_{z_t}}{r_t}\right]
    + \left[\sum_{s = 1}^S  \max_{z \in \mathcal{Z}} \sum_{t = \tau_{s-1}}^{\tau_s - 1} \E{z_t, \hat{\pi}_z}{r_t}
     - \sum_{t = 1}^T \E{z_t, \mathcal{E}_t}{r_t}\right],
\end{align*}
where we utilize the fact that each stationary segment has one optimal sub-policy.
\end{proof}

\mainregret*
\begin{proof}
We have the following regret decomposition due to \cref{lem:regret_decomposition},
\begin{align*}
    \mathcal{R}(T; \mathcal{E}, \hat{\Pi}) 
    &\leq \left[\sum_{t = 1}^T \E{z_t, \pi^*_{z_t}}{r_t}
    - \sum_{t = 1}^T \E{z_t, \hat{\pi}_{z_t}}{r_t}\right]
    + \left[\sum_{s = 1}^S  \max_{z \in \mathcal{Z}} \sum_{t = \tau_{s-1}}^{\tau_s - 1} \E{z_t, \hat{\pi}_z}{r_t}
     - \sum_{t = 1}^T \E{z_t, \mathcal{E}_t}{r_t}\right].
\end{align*}

The first term can be bounded using our offline analysis, which shows near-optimality of $\hat{\Pi}$ when the latent state is known. In the case where $z_{1:T}$ is the same both offline and online, we see that for each round $t$, $\E{z_t, \pi^*_{z_t}}{r_t} - \E{z_t, \hat{\pi}_{z_t}}{r_t} = V_t(\pi_{z_t}^*) - V_t(\hat{\pi}_{z_t})$. Hence, the first term is exactly $V(\Pi^*) - V(\hat{\Pi})$ and is bounded by \cref{thm:main} w.p. at least $1 - \delta_1 - \delta_2$.
The second term is the switching regret of Exp4.S, and is bounded by choosing $\ell = T/S$ in \cref{lem:exp4}. Combining the two bounds yields the desired result. 
\end{proof}

\end{document}